\documentclass[runningheads,a4paper]{llncs}

\usepackage{algorithm}
\usepackage{algpseudocode}

\usepackage{amsmath}

\usepackage{enumerate,amsmath,amssymb,graphicx}
\usepackage{url}
\usepackage{fancyvrb}
\usepackage{graphicx}
\usepackage{epsfig}
\usepackage{tikz,xspace,multicol}
\usepackage{soul}
\usepackage{mathtools}
\usepackage{bm}

\usepackage{enumitem}

\usetikzlibrary{positioning,shapes,arrows,automata}
\tikzset{ state/.style={draw,ellipse,initial text=} }

\tikzset{
  loop above right/.style={above right, out= 60, in= 30, loop},
  loop above left/.style ={above left,  out=150, in=120, loop},
  loop below right/.style={below right, out=330, in=300, loop},
  loop below left/.style ={below left,  out=240, in=210, loop}
}

\newcommand{\ctl}	{{\sf CTL}}

\newcommand{\ltl}	{{\sf LTL}}
\newcommand{\opacpsl}	{{\sf oPSL}}
\newcommand{\psl}{\textsf{PSL}}
\newcommand{\patl}	{{\sf PATL}}

\newcommand{\ie}    	{i.e., }

\newcommand{\true}   	{\ensuremath{{\mathtt{true}}}}

\newcommand{\obs}   	{{\mathsf{obs}}}

\newcommand{\last}   	{\mathit{last}}

\newcommand{\dist}   	{{\sf{Dist}}}

\def\A{\mathbf{A}}
\def\F{\mathbf{F}}
\def\G{\mathbf{G}}
\def\P{\mathbf{P}}

\def\U{\mathbf{U}}

\def\X{\mathbf{X}}

\def \opac{\odot}

\renewcommand{\vec}[1]{\boldsymbol{\mathrm{#1}}}

\newcommand{\RUN}[2]{\mathsf{Paths}_{#2}(#1)}

\newcommand{\HIST}[2]{\mathsf{Hist}_{#2}(#1)}

\newcommand{\ACT}{\mathsf{Act}}
\newcommand{\AG}{\mathsf{Ag}}
\newcommand{\ASP}{\mathsf{Ap}}

\newcommand{\var}{\mathsf{Var}}

\newcommand{\mycopy}{\tt{\scriptsize copy}}
\newcommand{\wait}{\tt{\scriptsize wait}}
\newcommand{\send}{\tt{\scriptsize send}}

\newcommand{\FS}{\rightarrow}
\newcommand{\CAL}[1]{\mathcal{#1}}

\newcommand{\DDD}{\mathbf{D}}

\newcommand{\GGG}{\CAL{G}}

\newcommand{\MMM}{\CAL{M}}

\newcommand{\VVV}{\CAL{V}}

\newcommand{\TRANS}[1]{\stackrel{#1}{\longrightarrow}}
\newenvironment{GRAMMAR}{\[\begin{array}{lcl}}{\end{array}\]}
\newcommand{\VERTICAL}{\  \mid\hspace{-3.0pt}\mid \ }

\def\funchelper#1,#2,#3{#1\colon#2\to#3}
\newcommand{\func}[1]{\funchelper#1}

\newcommand{\nmodels}{\centernot{\models}}
\newcommand{\paths}{\mathsf{Paths}}

\renewcommand{\models}{\mathrel{\Vdash}}
\renewcommand{\nmodels}{\mathrel{\nVdash}}

\newcommand{\defeq}{\coloneq}

\newcommand{\form}{\mathfrak{a}}

\urldef{\mailsa}\path|chunyan.mu@abdn.ac.uk, n.motamed@uu.nl, brian.logan@abdn.ac.uk, n.a.alechina@uu.nl|

\usepackage{etoolbox}

\usepackage{environ}

\newtoggle{draft}

\newcommand{\makeperson}[3]{\NewEnviron{#1}{\iftoggle{draft}{\begingroup\color{#3}\textbf{#2: }{\BODY}\endgroup}{}}}

\makeperson{nm}{Nima}{teal}
\makeperson{na}{Natasha}{blue}
\makeperson{bl}{Brian}{orange}
\makeperson{cm}{Chunyan}{brown}

\makeperson{removed}{Removed}{red}

\begin{document}

\title{Probabilistic Strategy Logic with \\ Degrees of Observability}
\titlerunning{Probabilistic Strategy Logic with Degrees of Observability}

 \author{Chunyan Mu\inst{1} \and Nima Motamed\inst{2} \and Natasha Alechina\inst{3,2} \and Brian Logan\inst{1,2}} 
 \authorrunning{C. Mu \and N. Motamed \and N. Alechina \and B. Logan} 
 \institute{$^1$ Department of Computing Science, University of Aberdeen, U.K.\\
 	$^2$ Information and computing sciences, Utrecht University, Netherlands\\
 	$^3$ Department of Computer Science, Open University Netherlands\\
 	{\small\tt chunyan.mu@abdn.ac.uk, n.motamed@uu.nl, brian.logan@abdn.ac.uk, n.a.alechina@uu.nl}
 }
 
\maketitle

\begin{abstract}
There has been considerable work on reasoning about the strategic ability of agents under imperfect information. However, existing logics such as Probabilistic Strategy Logic are unable to express properties relating to information transparency. \emph{Information transparency} concerns the extent to which agents' actions and behaviours are observable by other agents. Reasoning about information transparency is useful in many domains including security, privacy, and decision-making. In this paper, we present a formal framework for reasoning about information transparency properties in stochastic multi-agent systems.  We extend Probabilistic Strategy Logic with new observability operators that capture the degree of observability of temporal properties by agents. We show that the model checking problem for the resulting logic is decidable.
\end{abstract}

\section{Introduction}
\label{sec:intro}
Multi-Agent Systems (MASs) often involve agents that operate autonomously and interact with each other in dynamic and sometimes adversarial ways. Understanding the transparency and observability of these interactions is crucial for ensuring secure, efficient, and cooperative behaviour.
In particular, information transparency and agent observability directly impact the security and privacy of MASs in which agents share information and where there is a risk of unintentional data leakage. Analysing information transparency helps identify potential sources of data leakage and design mechanisms to prevent it. The ability to control what agents can observe and the information they can induce is crucial for safeguarding sensitive data and preventing information leakage.
In addition, the decision-making processes of agents are influenced by the information they possess about each other's behaviours and intentions. Quantified analysis of information transparency and agent observability plays a key role in determining the accuracy and effectiveness of decision-making within MASs.

This paper addresses the challenge of specifying, verifying and reasoning about information transparency properties within MASs. In particular, we specify observability properties from a standpoint of information transparency within the \textit{opacity} framework \cite{Bryans//:05a}. A property $\Phi$ is considered to be opaque, if for every behaviour $\pi$ satisfying $\Phi$ there is a behaviour $\pi'$ violating $\Phi$ such that $\pi$ and $\pi'$ are observationally equivalent, so that an observer can never be sure if $\Phi$ holds or not. Opacity (and its negation, observability) are so called \emph{hyperproperties} that relate multiple execution traces. Verification of hyperproperties is an emerging and challenging topic, see, for example, \cite{Beutner/Finkbeiner:23a}. 

In order to reason about the observability of behaviours, we introduce \emph{Opacity Probabilistic Strategy Logic} (\opacpsl), an extension of Probabilistic Strategy Logic (\psl) which allows quantitative analysis of information transparency. \opacpsl\ enables us to specify the degree of transparency in system behaviours to an observer under a binding of agents' strategies, taking into account predefined observability of behaviours for the observer. We use a very general
approach to observability which allows us to reason both about observability of state properties, and observability of actions. Using the concept of observability allows, for example, the identification of cases where agents observe sensitive information, disclosing potential information leakage or unauthorised access attempts. 

We introduce a novel framework for systematically analysing information transparency in stochastic multi-agent systems. The key contributions include:
\begin{itemize}
\item A definition of agent observability and information transparency concerning agent behaviours in the context of partially observable stochastic multi-agent systems with concurrent and probabilistic behaviours.

\item The introduction of Opacity Probabilistic Strategy Logic (\opacpsl) incorporating a new 
observability and degree of observability operators, allowing precise representation of observability properties that are challenging to express in standard logics for MAS.

\item Showing that the model-checking problem  for \opacpsl\ with memoryless strategies is decidable in 3EXPSPACE. 
\end{itemize}
The framework facilitates formal reasoning about agent observability and information transparency analysis in MAS, with applications in security, privacy, game theory, and AI.

The remainder of this paper is organised as follows. In Section \ref{sec:related} we discuss related work. In Section \ref{sec:model} we formalise stochastic multi-agent systems as partially observable stochastic models. Section \ref{sec:logic} extends Probabilistic Strategic Logic (PSL), incorporating new operators for quantified observability analysis. Section \ref{sec:veri} presented algorithmic approaches for systematic observability assessment within the PSL model checking framework. In Section \ref{sec:conclusion} we conclude and discuss possible directions for future work. 

\section{Related Work}
\label{sec:related}

There has been a significant amount of work on logical specification and formal verification of MAS with imperfect information.
Various methods have been developed to support the model checking of such logics, such as those focusing on epistemic logics for knowledge about the state of the system, e.g., \cite{LomuscioQR17,BelardinelliLMR17,KwiatkowskaN0S19,BelardinelliLMR20,BallotMLL24,MalvoneMS17}. 
Actions and knowledge in strategic settings has been investigated in \cite{Schobbens04,Goranko/Jamroga:04a,Agotnes:06a}. 
These approaches focus on formulas true in all  states or in all histories indistinguishable from the current one. In contrast, we are interested in properties that are either true on all observable traces, or false on all observable traces. In this we follow \cite{Bryans//:05a} which introduced the \emph{opacity} framework, which is based on being able to distinguish globally observable or unobservable properties. 

Probabilistic Alternating-Time Temporal Logic (\patl*) was proposed in \cite{Chen/Lu:07a}.
Huang et al. \cite{HuangSZ12} proposed an incomplete information version of \patl*, providing a framework for reasoning about possible states and actions based on agents' beliefs and strategies. However, it does not directly address information exposure resulting from the observation of agents' actions and behaviours or the security of information flows. Mu and Pang \cite{Mu23} developed a framework for specifying and verifying opacity in PATL. However, we propose an enhanced framework based on PSL, which offers greater expressiveness and flexibility.

Strategy logic with imperfect information~\cite{BerthonMMRV21} extends traditional modal logic to reason about strategic interactions in situations of imperfect information. It captures agents' decision-making abilities taking into account their limited knowledge. Additionally, work by \cite{FerrandoM23} addressed the verification of logics for strategic reasoning in contexts with imperfect information and perfect recall strategies, employing sub-model generation and Computation Tree Logic (\ctl*) model checking. A probabilistic extension of Strategy Logic (\psl) was introduced by \cite{AminofKMMR19} for stochastic systems; this work also investigated model-checking problems for agents with perfect and imperfect recall.

Our work is relevant to research on information flow security awareness analysis and verification, and also contributes to the quantified information flow security landscape. While previous studies focused on imperative modelling languages and probabilistic aspects \cite{AlvimCMMPS20} in secure computing systems, there has been little consideration of the dynamic collaboration, interaction, and decision-making patterns found in multi-agent scenarios. In contrast, this paper explores observability issues in MAS from a novel perspective of information transparency, and focuses on quantified information flow security awareness for MAS.

\section{Partially Observable MAS}
\label{sec:model}
In this section, we introduce \emph{partially observable multi-agent systems} which constitute the formal basis of our approach.

We write $\dist(X)$ for the set of all discrete probability distributions over a set $X$, 
\ie all functions $\mu: X \to \lbrack 0,1 \rbrack$ s.t. $\sum_{x \in X} \mu(x)=1$.
Given a set $X$, we denote by $2^X$ the power set of $X$, by $X^*$ the set of all finite words over $X$,
and by $X^\omega$ the set of all infinite words over $X$.

\begin{definition}
\label{def:pgs}
A \emph{stochastic transition system} (POMAS) is a tuple $\GGG=(\AG, S, \ACT, T)$, where
\begin{itemize}
    \item $\AG=\{1, \dots, n\}$ is a finite set of \emph{agents};
    \item $S$ is a finite set of \emph{states};
    \item $\ACT$ is a finite set of \emph{actions};
    \item $T: S \times \ACT^{\AG} \FS \dist(S)$ is a \emph{transition function}. 
\end{itemize}
\end{definition}
\noindent
At each time step, the agents simultaneously choose actions from $\ACT$, producing an \emph{action profile} $\alpha = (a^1, \dots, a^n) \in \ACT^\AG$. The transition function $T$ specifies for each state $s$ and such an action profile $\alpha$, a probability distribution $T(s, \alpha) \in \dist(S)$ over \emph{next} states: $T(s, \alpha)(s')$ is the probability of moving to $s'$ from $s$ when the agents choose $\alpha$.
We write $s \TRANS{\alpha} s'$ whenever $T(s, \alpha)(s') > 0$.

\begin{definition}
Given a stochastic transition system $\GGG$, a \emph{$\GGG$-\emph{path}} is an infinite sequence $\pi = (s_0, \alpha_0) (s_1 \alpha_1) \dots$ of pairs of states $s_i$ and action profiles $\alpha_i$, such that $s_i \TRANS{\alpha_i} s_{i+1}$ for all $i$. We will usually write $\pi = s_0 \alpha_0 s_1 \alpha_1 \cdots$ for simplicity of notation. The state $s_0$ is referred to as the \emph{initial state} of $\pi$.
We denote by $\pi_\mathsf{s}(j)$ the $j$th state $s_j$ of $\pi$,
and by $\pi_{\mathsf{a}}(j)$ the $j$th action profile $\alpha_j$ of $\pi$.
Furthermore, we denote by $\pi_{\le j}$ the prefix of $\pi$ up to the $j$th state,
i.e. $\pi_{\le j} = s_0 \alpha_0 s_1 \dots s_j$.
Given a state $s$, we write $\RUN{s}{\GGG}$ for the set of all $\GGG$-paths $\pi$ with initial state $s$.
\end{definition}

\noindent
We also define \emph{histories} as \emph{finite} sequences $h = s_0 \alpha_0 \cdots s_n$, such that $s_i \TRANS{\alpha_i} s_{i+1}$ for all $i < n$.
We denote the set of all histories $h$ with initial state $s$ by $\HIST{s}{\GGG}$. Given a history $h = s_0 \alpha_0 \cdots s_n$, we write $\last(h)$ for its final state $\last(h) = s_n$.

\begin{example}
\label{eg:sts}
Consider the simple message interception scenario illustrated in Figure ~\ref{fig:eg1}.
\begin{figure}
\begin{center}
\includegraphics[width=0.6\textwidth]{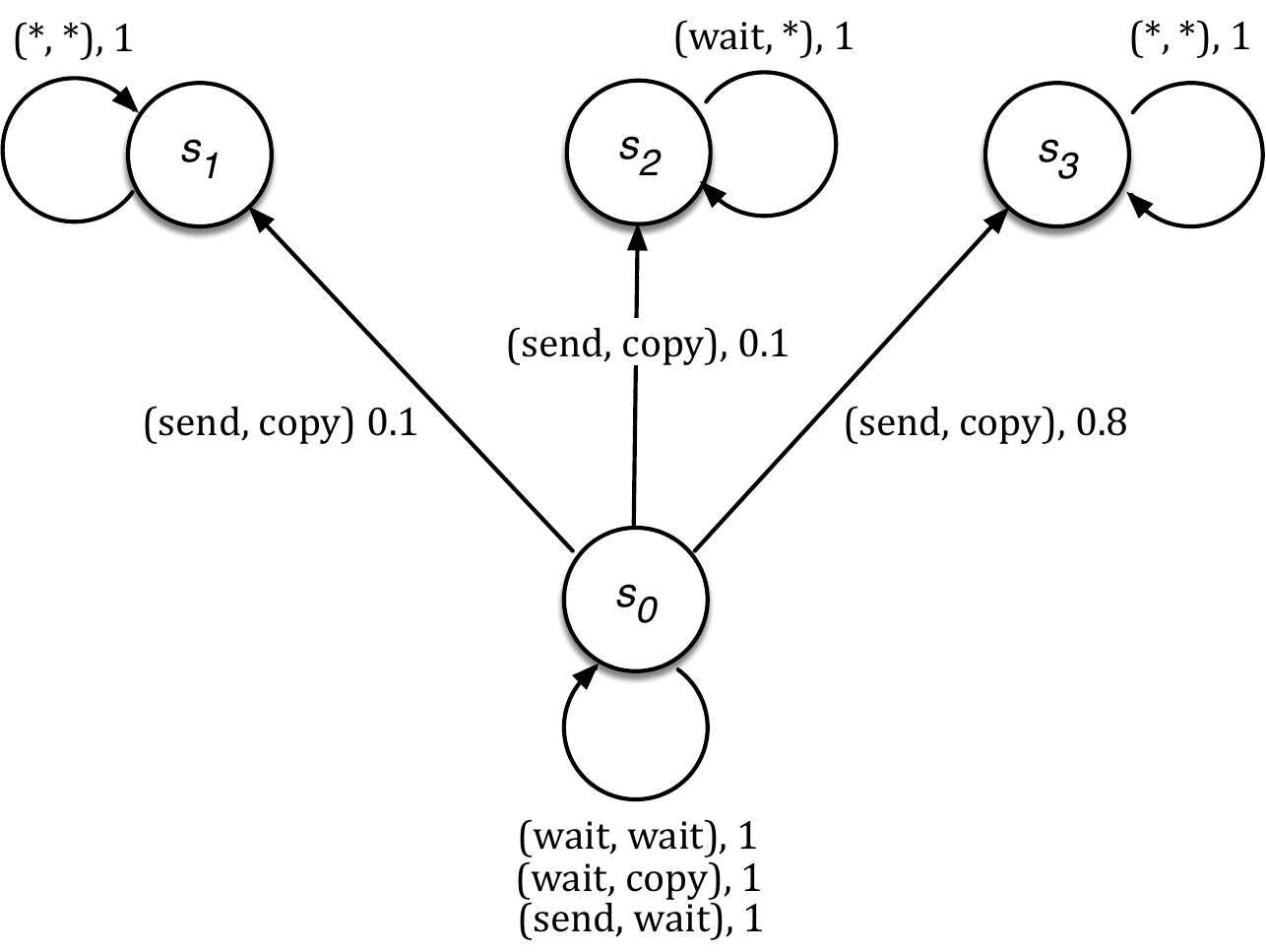}
\caption{Simple message interception scenario}
\end{center}
 \label{fig:eg1}
\end{figure}

We have:
$\AG = \{ 1,2 \}$,
$\ACT_1 = \{ \send, \wait \}$,
$\ACT_2 = \{ \mycopy, \wait \}$,
$S = \{ s_0,s_1, s_2,s_3 \}$.

Example paths from $s_0$ include:
\begin{align*}
\pi &= s_0 \TRANS{(\send,\mycopy)} s_1 \TRANS{(\send,\mycopy)}  s_1 \TRANS{(\send,\mycopy)} s_1 \ldots \\
\pi' &= s_0 \TRANS{(\send,\mycopy)} s_3 \TRANS{(\send,\mycopy)}  s_3 \TRANS{(\send,\mycopy)} s_3 \ldots
\end{align*}

Example histories include the following prefixes of $\pi$ and $\pi$':
\begin{align*}
h &= s_0 \TRANS{(\send,\mycopy)} s_1 \TRANS{(\send,\mycopy)} s_1 \\
h' &= s_0 \TRANS{(\send,\mycopy)} s_3 \TRANS{(\send,\mycopy)} s_3 \\
\end{align*}
\end{example}

\paragraph{Observation of behaviours}
To model the observational capabilities of an agent, we use a set of \textit{observables}, distinct from the states and actions of the transition system. 
We consider two types of observables: \emph{state} and \emph{action} observables. We denote by $\Theta_\mathsf{s}$ the finite set of state observables, and by $\Theta_\mathsf{a}$ the finite set of action observables. The sets $\Theta_\mathsf{s}$ and $\Theta_\mathsf{a}$ contain distinguished `invisible' state/action observables $\circ$ and $\epsilon$, respectively. Observables are then defined as pairs of action and state observables, i.e., $\Theta \defeq \Theta_\mathsf{a} \times \Theta_\mathsf{s}$. An \emph{observation function} is a pair $(\obs^\mathsf{a}, \obs^\mathsf{s})$ of functions $\func{\obs^\mathsf{a}, \ACT^\AG, \Theta_\mathsf{a}}$ and $\func{\obs^\mathsf{s}, S, \Theta_\mathsf{s}}$. For notational convenience, we combine these into a single function $\func{\obs, \ACT^\AG \times S, \Theta}$ (i.e., with $\obs(\mathbf{\alpha}, s) = (\obs^\mathsf{a}(\alpha), \obs^\mathsf{s}(s))$. 

We will often lift observation functions to functions $\func{\obs, \mathsf{Paths}_\GGG(s), \Theta^\omega}$ operating on paths by letting $\obs(s_0 \alpha_0 s_1 \cdots) = \epsilon\ \obs^\mathsf{s}(s_0) \obs^\mathsf{a}(\alpha_0) \obs^\mathsf{s}(s_1) \cdots$, where $\epsilon$ is the distinguished action observable, and lift $\obs$  to sets of paths $X \subseteq \mathsf{Paths}_\GGG(s)$ by letting $\obs(X) = \{\obs(\pi) \mid \pi \in X\}$.
Finally, given a set $X \subseteq \mathsf{Paths}_\GGG(s)$, we write $\obs^\sim (X)$ for the set of all $\pi \in \mathsf{Paths}_\GGG(s)$ for which there exists $\pi' \in X$ with $\obs(\pi) = \obs(\pi')$. In other words, $\obs^\sim(X)$ is the set of all paths that are observationally equivalent to some path in $X$.

\begin{definition}
A \emph{partially observable multi-agent system} is a tuple 
$\MMM=(\GGG, \ASP, L, \{\obs_i\}_{i \in \AG})$, where:
\begin{itemize}

\item $\GGG=(\AG, S, \ACT, T)$ is a stochastic transition system;

\item $\ASP$ is a finite set of \emph{atomic propositions}; 

\item $L: S \to 2^{\ASP}$ is a \emph{state labelling function};

\item $\func{\obs_i, \ACT^\AG \times S, \Theta}$ is the observation function of $i$.
\end{itemize}
\end{definition}

\begin{example}
\label{eg:model}
The following partially observable multi-agent system models the running example introduced in Example~\ref{eg:sts}.

Let the labelling $L$ be as follows: $L(s_0) = \{\mathit{init}\}$, $L(s_1) = \{\mathit{stolen}\}$,
$L(s_2) = \{\mathit{stolen, warning}\}$ and $L(s_3) = \emptyset$.

Let the observation function of agent 1 be as follows:
\begin{align*}
\obs^s_1 (s_0) &= \mathit{init} \\
\obs^s_1 (s_1) &= \obs^s_1 (s_3) = \circ \\
\obs^s_1 (s_2) &= \mathit{warning} \\
\obs^a_{1}(\send, \wait) &= \obs^a_{1}(\send, \mycopy) = (\send, \epsilon)\\
\obs^a_{1}(\wait, \wait) &= \obs^a_{1}(\wait, \mycopy) = (\wait, \epsilon)
\end{align*}
\noindent
Observations of agent 2 are perfect, that is $\obs^s_2(s_i)=s_i$ and $\obs^a_2(\alpha_i) = \alpha_i$ for every state $s_i$ and joint action $\alpha_i$.

We then have:
$$\obs_{1} (h) = \obs_{1} (h') =  \mathit{init} \TRANS{(\send,\epsilon)} \circ \TRANS{(\send,\epsilon)} \circ$$
This implies that agent $1$ is not able to distinguish $h$ and $h'$ under $\obs_1$. The same holds for the paths $\pi$ and $\pi'$: they are observationally equivalent for agent $1$, $\pi' \in  \obs^\sim(\{\pi\})$.
\end{example}

Agent interactions are governed by strategies informed by their observations. 
These strategies are often referred to as \emph{uniform strategies}: they do not rely directly on the partially observable multi-agent system, but instead on the observables the model produces.
\begin{definition}
A \emph{strategy} is a function $\func{\sigma, \Theta^\ast, \dist(\ACT)}$.
We write $\Sigma$ for the set of all possible strategies. 
A \emph{memoryless strategy} is a function $\sigma: \Theta_\mathsf{s} \to \dist(\ACT)$, \ie the decision made by the strategy is solely dependent on the observation of the current state, and not on the entire history of observations.
A \emph{strategy profile} (resp. \emph{memoryless strategy profile}) is a tuple $\vec{\sigma}=(\sigma_1, \sigma_2, \dots, \sigma_n) \in \Sigma^\AG$ of strategies (resp. memoryless strategies) $\sigma_i$ for each agent $i$.
\end{definition}

Note that memoryless strategy profiles allow us to transform a partially observable multi-agent system and strategy profile into Markov chain, which we consider to consist of a set of states $X$ together with a transition function $\func{t, X, \dist(X)}$. Given a partially observable multiagent system $\mathcal{M}$ (with set of states $S$ and transition function $T$), and a memoryless strategy profile $\vec{\sigma}$, we define the Markov chain $M_{\vec{\sigma}} = (S, t_{\vec{\sigma}})$ by letting $$t_{\vec{\sigma}}(s_1)(s_2) = \sum_{\alpha \in \ACT^\AG} T(s_1, \alpha)(s_2) \prod_{i \in \AG} \sigma_i(\obs^\mathsf{s}_i(s_1))(\alpha_i).$$

\section{The Logic \opacpsl}
\label{sec:logic}
Strategy Logic~\cite{ChatterjeeHP10} and Probabilistic Strategic Logic (\psl)~\cite{AminofKMMR19} are formal languages designed for reasoning about multi-agent systems, where autonomous agents make strategic decisions in an environment. These logics focus on capturing the strategic interactions among agents in a multi-agent system, where agents are viewed as rational entities capable of reasoning about their actions and the actions of others. \psl\ extends Strategy Logic to incorporate uncertainty and probabilistic reasoning, and was designed for scenarios where agents operate in an environment with inherent randomness.
\psl\ formulas are interpreted on multi-agent stochastic transition systems where transitions result from the concurrent actions of agents. An agent's strategy determines the probability that the agent will select a given action in any given situation, typically based on the system's historical evolution. 

We extend \psl~\cite{AminofKMMR19} to allow reasoning about observability operators and strategic abilities of agents. This syntax allows us to specify a wide range of properties and conditions in \opacpsl, including those related to observability from the perspectives of different agents. It also supports probability-related operations for modelling probabilistic behaviours in MAS. The key addition is the \emph{observability formulae} $\opac_{i} \Phi$ and the \emph{degree of observability} terms $\DDD_{\beta, i}(\Phi)$.

Before we introduce the syntax, we first define \emph{bindings} and \emph{valuations}, which are used in the syntax to determine which strategies are used by which agents.

\begin{definition}
Let $\var$ be a set of \emph{variables}. A \emph{binding} is a function $\beta: \AG \to \var$ assigning variables to agents.  
A \emph{valuation} is a \emph{partial} function $\nu: \var \rightharpoonup {\Sigma}$ assigning strategies to some variables. Given a valuation $\nu$, a strategy $\sigma$, and a variable $x$, we denote by $\nu\lbrack x \mapsto \sigma \rbrack$ the valuation defined as
$\nu\lbrack x \mapsto \sigma \rbrack(x) \defeq \sigma$ and $\nu\lbrack x \mapsto \sigma \rbrack(y) \defeq \nu(y)$ for all $y \neq x$.
\end{definition}

Intuitively, bindings tell agents which strategies, represented by variables, they should follow, while valuations determine which strategies variables refer to. Note that if the range of $\beta$ is contained in the domain of $\nu$ then composing them produces a strategy profile $\nu \circ \beta \in \Sigma^\AG$.

\begin{definition}
\label{def:syntax}                                                       
The \emph{syntax} of \opacpsl\ includes three classes of formulae: 
\emph{history formulae}, 
\emph{path formulae}, and
\emph{arithmetic terms}
ranged over by $\varphi$, $\Phi$, and $\tau$ respectively.
\begin{GRAMMAR}
  \varphi
     &::=&
  p 
     \VERTICAL
  \neg \varphi
     \VERTICAL
  \varphi \lor \varphi
     \VERTICAL
  \exists x. \varphi
     \VERTICAL
  \tau \bowtie \tau
     \VERTICAL
  \opac_{i} \Phi
  	\\
  \Phi
     &::=&
  \varphi
  	\VERTICAL
  \neg \Phi
    \VERTICAL
  \Phi \lor \Phi    
    \VERTICAL  
  \X \Phi 
     \VERTICAL
  \Phi  \U \Phi 
     \\
  \tau &::=&
  	c 
  	\VERTICAL
        \tau^{-1}
        \VERTICAL
  	\tau \oplus \tau
  	\VERTICAL
  	\P_{\beta}(\Phi)
  	\VERTICAL
  	\DDD_{\beta, i} (\Phi)
\end{GRAMMAR}
\noindent
where $p \in \ASP$ is an \emph{atomic proposition}, 
$x \in \var$ is a variable,
${\bowtie} \in \{>,<,=\}$,
$i \in \AG$,
$c \in \mathbb{Q}$,
$\oplus \in \{ +,\times\}$,
and $\beta: \AG \to \var$ is a binding.
\end{definition}
As for \psl, in \opacpsl\ the variables $x \in \mathsf{Var}$ represent variations in agent strategies and are quantified over in state formulas like $\exists x.\varphi$. The path formulas contain standard \ltl\ temporal operators: $\X \Phi$ means `$\Phi$ holds in the next step', and $\Phi \U \Psi$ means `$\Phi$ holds until $\Psi$ holds'. The arithmetic term $\P_{\beta} (\Phi)$ indicates the probability that, given agent $i$ employs strategy $\beta (i)$ for each $i \in \AG$, a random path in the system satisfies $\Phi$. 

The key addition to the logic is the observability operator and the degree of observability. The observability operator $\opac_{i} \Phi$ states that the behaviours satisfying $\Phi$ are observable to agent $i \in \AG$. Specifically, it requires that for each path $\pi$ satisfying $\Phi$, there is no observationally equivalent path $\pi'$ violating $\Phi$ from agent $i$'s perspective. This operator allows us to assess the observational capabilities of agent $i$ regarding system behaviours expressed by property $\Phi$. The quantitative observability term $\DDD_{\beta, i} (\Phi) $ expresses the degree of observability of agent $i$ with respect to $\Phi$-paths under strategy bindings $\beta$, which is defined as the probability of having a $\Phi$-path that is not observationally equivalent to a $\lnot \Phi$-path, conditional on having a $\Phi$-path. 

Before we provide the semantics, we briefly comment on how probabilities of sets of paths are calculated: the following is largely standard in the literature on probabilistic infinite-trace temporal logics, but we restate it hear to clarify the notation. 
Given a history $h$ with initial state $s$ in $\GGG$, the \emph{cone} of $h$, denoted by $\langle h \rangle$, is the set of all paths $\pi$ with initial state $s$ that extend $h$, i.e., such that $\pi = h\alpha\pi'$ for some action profile $\alpha$ and path $\pi'$. Taking a strategy profile $\vec{\sigma}$ and history $h$, we consider a probability measure $\mu_{h, \vec{\sigma}}$ on the $\sigma$-algebra generated by all cones of histories $h'$ with the initial state $\last(h)$. This probability measure is the unique measure such that for all histories $h' = s_0 \alpha_0 \cdots s_n$, we have
\begin{align*}
\mu&_{h, \vec{\sigma}}(\langle h' \rangle)
&= \prod_{k < n} T(s_k, \alpha_k)(s_{k+1}) \cdot \prod_{i \in \AG}\sigma_i(\obs_i(h_{-1} h'_{\le k}))(\alpha_k^i),
\end{align*}
where $h_{-1}$ is $h$ with the final state $\last(h)$ cut off (since otherwise the state $\last(h)$ would appear twice successively in $hh'_{\le k}$).
Intuitively, this is the probability of encountering the history $h'$ after $h$ when everyone acts according to $\vec{\sigma}$. As with other similar logics, it is a standard exercise to verify that the sets of paths encountered in the semantics are all elements of the $\sigma$-algebra we consider.

\begin{definition}
\label{def:semantics}
Given a partially observable multiagent system $\MMM=(\GGG, \ASP, L, \{\obs_i\}_{i \in \AG})$, 
a valuation $\nu: \var \rightharpoonup \Sigma$, and
a binding $\beta: \AG \to \var$, the
semantics for \opacpsl\ is defined by simultaneous induction over history formulae, arithmetic terms and path formulas as follows.

For a history $h$, we define:

\begin{itemize}

  \item $h,\nu\models p$ iff $p \in L(\last(h))$.

  \item $h,\nu \models \neg\varphi$ iff $h,\nu \nmodels \varphi$.

  \item $h,\nu \models \varphi \lor \psi$ iff 
  	$h,\nu \models \varphi$ or  $h,\nu \models  \psi$.
  	
  \item $h,\nu \models \exists x.\varphi$ iff 
  	there exists $\sigma \in \Sigma$ such that $h,\nu \lbrack x \mapsto \sigma \rbrack \models \varphi$.

 \item $h,\nu \models \tau \bowtie \tau'$ iff 
	$\VVV_{h,\nu}(\tau) \bowtie \VVV_{h,\nu}(\tau') $
	where: 
        \begin{itemize}
        \item $\VVV_{h,\nu}(c) = c$,
        \item $\VVV_{h, \nu}(\tau^{-1}) = (\VVV_{h, \nu}(\tau))^{-1}$,
 	\item $\VVV_{h,\nu}(\tau \oplus \tau') = \VVV_{h,\nu}(\tau) \oplus \VVV_{h,\nu}(\tau')$,
 	\item $\VVV_{h,\nu}(\P_{\beta}(\Phi))= \mu_{h, \nu \circ \beta} (\{\pi \mid \pi,\nu, 0 \models \Phi \})$,
        \item $\VVV_{h,\nu}(\DDD_{\beta, i}(\Phi))
        = \frac{\mu_{h, \nu \circ \beta} (\{\pi \mid \pi,\nu, 0\ \models\ \Phi \}\ \setminus\ \obs_i^\sim (\{\pi' \mid \pi',\nu, 0\ \nmodels\ \Phi\}) )}{\mu_{h, \nu \circ \beta}(\{\pi \mid \pi, \nu, 0\ \models\ \Phi\})}$.
        \end{itemize}

\item $h,\nu \models_{\MMM} \opac_{i} \Phi$ iff 
 	for all $\pi, \pi' \in \mathsf{Paths}_\GGG(\last(h))$ it holds that if $\pi, \nu, 0 \models \Phi$ and $\pi', \nu, 0 \nmodels \Phi$, then $\obs_i(\pi) \ne \obs_i(\pi')$.
    
\end{itemize}
\noindent
For a path $\pi$ of $\GGG$ and $k \geqslant 0$, we define:

\begin{itemize}

\item $\pi, \nu, k \models \varphi$ iff $\pi_{\le k},\nu \models \varphi$.

\item $\pi, \nu, k \models \neg \Phi$ iff $\pi,\nu, k \nmodels \Phi$.

\item $\pi, \nu, k \models \Phi \lor \Psi$ iff $\pi,\nu, k \models \Phi$ or $\pi,\nu, k \models \Psi$.

\item $\pi, \nu, k \models \X \Phi$ iff $\pi,\nu, k+1 \models \Phi$.
  
\item $\pi,\nu, k \models \Phi \U \Psi$ iff there exists $\ell \ge k$ such that $\pi,\nu,\ell \models \Psi$ and for all $m \in \lbrack k, \ell)$ we have $\pi,\nu,m \models \Phi$.
 
\end{itemize}

\end{definition}
\noindent
As in \psl, we use the convention that for all $h$, $\nu$, and formulas $\tau \bowtie \tau'$ containing either a subterm $\rho^{-1}$ such that $\rho$ is evaluated to 0, or a subterm $\mathbf{D}_{\beta, i}(\Phi)$ for which the value in the denominator is evaluated to 0, we put $h, \nu \nmodels \tau \bowtie \tau'$ by default, to avoid issues with division by zero.

A variable $x$ is \emph{free} in an \opacpsl\ formula if it appears in the domain of a binding $\beta$ appearing within a subformula $\P_{\beta}(\Phi)$ or term $\DDD_{\beta, i}(\Phi)$ that is not within the scope of a quantifier $\exists x$. A history formula is a \emph{sentence} if it contains no free variables.

As usual, the until operator allows to derive the temporal modality $\F$ (``eventually'') and $\G$ (``always''): $\F\Phi \ {\buildrel\rm def\over=} \ \true~ \U~ \Phi$ and $\G\Phi \ {\buildrel\rm def\over=} \ \neg\F \neg\Phi$. 

In the context of the \opacpsl\ semantics, if we restrict our attention to \textit{memoryless} strategies, we can consider the semantics to be interpreted over states $s$ instead of histories $h$. 

\begin{example}
\label{eg:logic}
Continuing with the partially observable multi-agent system from Example \ref{eg:model}.

Let us consider the property $\F \mathit{stolen}$. It is not observable by agent $1$ given a history/state $s_0$. This is because as we saw before, there are two paths $\pi$ and $\pi'$ from $s_0$, one of which (going through $s_1$) satisfies $\F \mathit{stolen}$, and the other one (going through $s_3$) does not. So: 
$$h,\nu \not \models_{\MMM} \opac_{1} (\F \mathit{stolen})$$ 
where $h = s_0$ and $\nu$ is any assignment.
Clearly, 
$$h,\nu \models_{\MMM} \opac_{2} (\F \mathit{stolen})$$ 
because agent $2$ can observe all states and can always distinguish $\F \mathit{stolen}$-paths from $\neg \F \mathit{stolen}$-paths.

Let us consider the degree of observability of $\F \mathit{stolen}$ given the same history and a
binding $\beta$ that assigns strategies $(\sigma_1,\sigma_2)$ to the agents. $\sigma_1$ is as follows:
it requires $1$ to perform $\send$ with probability $1$ in $\mathit{init}$, $\wait$ with probability $1$ in $\mathit{warning}$, and perform $\send$ and $\wait$ with probability $0.5$ in $\circ$. 

$\sigma_2$ requires agent $2$ to perform $\mycopy$ in $s_0$, and perform $\mycopy$ and $\wait$ with probability $0.5$ in $s_1, s_2$ and $s_3$.

Recall that $\VVV_{h,\nu}(\DDD_{\beta, 1}(\F\mathit{stolen}))$ is

        $$\frac{\mu_{h, \nu \circ \beta}([\F \mathit{stolen}]\ \setminus\ \obs_1^\sim ([\neg \F \mathit{stolen}]) )}{\mu_{h, \nu \circ \beta}([\F \mathit{stolen}])}$$
where we denote by $[\F \mathit{stolen}]$ the set of paths $\{\pi \mid \pi, \nu, 0\ \models\ \F \mathit{stolen}\}$, similarly for $[\neg \F \mathit{stolen}]$.

The set of paths $[\F \mathit{stolen}]$ contains paths going through $s_1$
and $s_2$, that is, paths with prefixes $s_0 \TRANS{(\send,\mycopy)} s_1$ and
$s_0 \TRANS{(\send,\mycopy)} s_2$. $\mu_{h, \nu \circ \beta}$ assigns this set $0.2$.

The set of paths $[\neg \F \mathit{stolen}]$ contains paths going through $s_3$.
For agent $1$, who cannot distinguish $s_1$ and $s_3$, the set of paths observationally equivalent to $[\neg \F \mathit{stolen}]$, $\obs_1^\sim ([\neg \F \mathit{stolen}])$, contains all paths with prefixes $s_0 \TRANS{(\send,\mycopy)} s_1$ and $s_0 \TRANS{(\send,\mycopy)} s_3$.

The set of paths $[\F \mathit{stolen}]\ \setminus\ \obs_1^\sim ([\neg \F \mathit{stolen}])$ therefore contains paths with prefixes $s_0 \TRANS{(\send,\mycopy)} s_2$. $\mu_{h, \nu \circ \beta}$ assigns this set $0.1$.

Hence $\VVV_{h,\nu}(\DDD_{\beta, 1}(\F \mathit{stolen})) = \frac{0.1}{0.2} = 0.5$.
\end{example}

\begin{example} 
\label{eg:logic1}
Lastly, we consider a few more general examples. Let $\vec x= (x_1, \dots, x_n)$ be a tuple of $n$ strategy variables.
\begin{itemize}
\item [1)] Asking whether, against all possible strategies of the other agents $2, \dots, n$, agent $1$ has a strategy $x_1$ such that the path formula $\Psi$ is observable to her can be expressed as:
\[\exists x_1. \forall x_2. \forall x_3 \dots \forall x_n. \left(  \opac_{i} (\Psi) \right)\]

\item [2)] Asking whether any other agent $i \ge 2$ has a strategy $x_i$, such that for all possible strategies of agent $1$, the observability degree of the path formula $\Psi$ from $1$'s view is $\le 0.1$, provided that $\psi$ holds with probability $\ge 0.9$, can be expressed as:
\[\forall x_1. \exists x_2 \dots \exists x_n. \left(\P_{\vec{x}}(\Psi) \ge 0.9 \to \DDD_{\vec{x}, 1} (\Psi) \le 0.1 \right)\]
\end{itemize}
\end{example}

\section{Verification of \opacpsl}
\label{sec:veri}
Since \opacpsl\ extends \psl, its model checking problem is undecidable, even when restricted to partially observable multi-agent systems with only a single agent \cite{AminofKMMR19,BrazdilBFK06}.

However, if we restrict our attention to memoryless strategies, the model checking problem becomes decidable.

\begin{theorem}
Model checking \opacpsl\ sentences when the semantics is interpreted only with memoryless strategies, is decidable.
\end{theorem}

We describe the model checking procedure exhibiting decidability in the next section.

\subsection{Model Checking Algorithm}

In this section, we outline the model checking algorithm for \opacpsl\ when considering memoryless strategies. The complete procedure can be found in the Appendix. The basis of the procedure lies in that for \psl\ \cite{AminofKMMR19}, but the novel operators of \opacpsl\ present highly non-trivial challenges, as we will see. 

Given a partially observable multi-agent system $\MMM$, we will define first-order logic formulas $\form_{\varphi, s}$ by induction over state formulas $\varphi$ 
and states $s$, written in the signature of real arithmetic. The construction will be such that
for all \opacpsl\ \emph{sentences} $\varphi$ and states $s$, we have that
$\form_{\varphi, s}$ holds in the theory of real arithmetic if and only if $s \models \varphi$. Using the decidability of this theory, we then get a model checking algorithm.

Throughout we will often write $\top_{\mathit{cond}}$ for some metalogical condition $\mathit{cond}$, defined as $\top_\mathit{cond} \defeq \top$ if $\mathit{cond}$ is true, and $\top_\mathit{cond} \defeq \bot$ otherwise. We will denote the equality symbol inside of the real arithmetic formula by $\approx$ to avoid any confusion.

\paragraph{Atoms and Booleans}
Let $\form_{p, s} \defeq \top_{p \in L(s)}$. We put $\form_{\varphi \land \psi, s} \defeq \form_{\varphi, s} \land \form_{\psi, s}$ and $\form_{\lnot \varphi, s} \defeq \lnot \form_{\varphi, s}$.

\paragraph{Existential quantifier}
Given the formula $\exists x.\varphi$, we introduce variables $r_{x, \theta, a}$, intuitively encoding the probability that the strategy $x$ performs action $a$ upon observing $\theta \in \Theta^\mathsf{s}$. Let $\form_{\exists x.\varphi, s} \defeq \exists (r_{x, \theta, a})_{\theta \in \Theta^\mathsf{s}, a \in \ACT} [\dist_x \land \form_{\varphi, s}],$
where
$$
\dist_x \defeq [\bigwedge_{\theta \in \Theta^\mathsf{s}, a \in \ACT} r_{x, \theta, a} \geqslant 0] \land [\bigwedge_{\theta \in \Theta^\mathsf{s}}\sum_{a \in \ACT} r_{x, \theta, a} \approx 1].
$$
The formula $\dist_x$ encodes that the variables $r_{x,\theta,a}$ give probability distributions for each $\theta$.

\paragraph{Full observability formulas}
Given the formula $\odot_k \Phi$, we write a formula that expresses that it is \textbf{not} the case that $s \nmodels \odot_k \Phi$. The reason for this double negation is that we can relatively easily express $s \nmodels \odot_k \Phi$: this holds iff there exist paths $\pi, \pi' \in \paths_\MMM(s)$ such that $\pi \models \Phi$ and $\pi' \models \lnot \Phi$, but $\obs_k(\pi) = \obs_k(\pi')$. To check this, we will describe a rooted directed graph $\GGG_{\odot_k \Phi}$ inside of the real arithmetic formula, such that infinite walks through the graph (starting from the root) correspond precisely to pairs $\pi, \pi' \in \paths_\MMM(s)$. Furthermore, we will identify a set $F_{\odot_k\Phi}$ of vertices in the graph, such that infinite walks $(\pi, \pi')$ through the graph reach $F_{\odot_k \Phi}$ \emph{infinitely often} if and only if $\pi \models \Phi$, $\pi' \models \lnot \Phi$, and $\obs_k(\pi) = \obs_k(\pi')$. Given this, all we then need to express that $s \nmodels \odot_k \Phi$ is that there exists an infinite walk through the graph that hits $F_{\odot_k \Phi}$ infinitely often.

Note that $\Phi$ can be considered to be an LTL formula over atomic propositions $W = 2^{\mathsf{Max}(\Phi)}$, where $\mathsf{Max}(\Phi)$ is the set of maximal state subformulas of $\Phi$. So we can construct nondeterministic Büchi automata $A_\Phi = (Q_\Phi, D_\Phi, q_\Phi^\ast, F_\Phi)$ and $A_{\lnot\Phi} = (Q_{\lnot\Phi}, D_{\lnot\Phi}, q_{\lnot\Phi}^\ast, F_{\lnot\Phi})$ over alphabet $W$, recognizing those $\mathbf{w} \in W^\omega$ such that $\mathbf{w} \models_\mathsf{LTL} \Phi$ and $\mathbf{w} \models_\mathsf{LTL} \lnot \Phi$, respectively.
We finally define $\form_{\odot_k\Phi, s}$ by stating that we cannot reach vertices in $F_{\odot_k \Phi}$ from the root that go back to themselves with non-empty paths, which is an efficient way of stating no infinite walk through the graph hits $F_{\odot_k \Phi}$ infinitely often.

The detailed definitions of $\form_{\odot_k\Phi, s}$ are omitted here due to lack of space and can be found in the Appendix.

\paragraph{Inequalities}
Given the formula $\tau_1 \leqslant \tau_2$, we introduce variables $r_{\tau}$ for all arithmetic subterms $\tau$ appearing in $\tau_1$ and $\tau_2$ (we denote the set of all such subterms by $\mathit{Sub}(\tau_1, \tau_2)$), which will intuitively hold the values the terms $\tau$ will take at the state $s$ under consideration, and given the strategies assigned to the strategic variables $x$. We define $\form_{\tau_1 \leqslant \tau_2, s}$ as
$$
\exists (r_{\tau})_{\tau \in \mathit{Sub}(\tau_1, \tau_2)}[\mathit{Eqn}_{\tau_1, s} \land \mathit{Eqn}_{\tau_2, s} \land r_{\tau_1} \leqslant r_{\tau_2}],
$$
where the formula $\mathit{Eqn}_{\tau_i, s}$ encodes that the variable $r_{\tau_i}$ indeed holds the value of $\tau_i$ in $s$ under the current valuation. We will now define these formulas.

\paragraph{Arithmetic terms}
We let:
$$\mathit{Eqn}_{c, s} \defeq [r_c \approx c],$$ 
$$\mathit{Eqn}_{\tau^{-1}, s} \defeq \mathit{Eqn}_{\tau, s} \land [r_{\tau} \times r_{\tau^{-1}} \approx 1],$$ 
$$\mathit{Eqn}_{\tau_1 + \tau_2, s} \defeq \mathit{Eqn}_{\tau_1, s} \land \mathit{Eqn}_{\tau_2, s} \land [r_{\tau_1 + \tau_2} \approx r_{\tau_1} + r_{\tau_2}],$$
$$\mathit{Eqn}_{\tau_1 \times \tau_2, s} \defeq \mathit{Eqn}_{\tau_1, s} \land \mathit{Eqn}_{\tau_2, s} \land [r_{\tau_1 \times \tau_2} \approx r_{\tau_1} \times r_{\tau_2}],$$

\paragraph{Probabilistic terms}
Given the formula $\P_\beta (\Phi)$, we `internalize' the description of the model checking procedure of PCTL* inside of our real arithmetic formula.

Considering $\Phi$ to be an LTL formula over $W = 2^{\mathsf{Max}(\Phi)}$, as in the construction for the full observability formula, we construct a deterministic Rabin automaton $A_\Phi = (Q, q^\ast, \delta, \mathsf{Acc})$ over alphabet $W$ accepting those $\mathbf{w} \in W^\omega$ such that $\mathbf{w} \models_\mathsf{LTL} \Phi$.
Recall that a (memoryless) strategy profile $\bm{\sigma}$ (as e.g. given by a binding $\beta$ and valuation $\nu$) turns our model $\MMM$ into a Markov chain $M_{\bm{\sigma}} = (S, t_{\bm{\sigma}})$, with $\func{t_{\bm{\sigma}}, S, \dist(S)}$ defined by putting $$t_{\bm{\sigma}}(s_1)(s_2) = \sum_{\alpha \in \ACT^\AG} T(s_1, \alpha)(s_2) \prod_{i \in \AG} \sigma_i(\obs_i(s_1))(\alpha_i).$$

Inside of the real arithmetic formula, we wish to compute the probability that a random walk through this Markov chain starting at $s$ satisfies $\Phi$. To do so, we will internalize the product Markov chain construction as is used in PCTL* model checking. The product Markov chain $M_{\bm{\sigma}} \otimes A$ has the same dynamics as $M_{\bm{\sigma}}$, while also providing as input to the automaton $A$ the formulas in $W$ which are true at each state.
It follows from this construction and the defining property of the DRA $A$, that the probability we are after is precisely the probability that a random walk through $M_{\bm{\sigma}} \otimes A$ from $(s, q^\ast)$ is `accepted' by $A$, in the sense that the walk's sequence of automaton states $q_0 q_1 \cdots \in Q^\omega$ satisfies the Rabin acceptance condition $\mathsf{Acc}$.
Following the method in PCTL* model checking, the problem can be reduced to computing the \emph{reachability} probability of certain \emph{accepting} terminal strongly connected components, which can be done by solving a system of linear equations. 
This process can be (efficiently) encoded in real arithmetic, details can be found in the Appendix.
We only sketch the key formulas here.

We write a formula $\mathit{Goal}_v$ that efficiently expresses that a state $v$ is part of an accepting terminal strongly connected component, as well as a formula $\mathit{Sol}$, which expresses the solution to the system of linear equations. We can then write out $\mathit{Eqn}_{\P_\beta (\Phi), s}$:
\begin{eqnarray*}
\mathit{Eqn}_{\P_\beta (\Phi), s}
    &\defeq& \exists (r_{v}^{w}, r^\mathrm{sol}_{v})_{v, w \in S \times Q}. \\
    && \big[ \mathit{Prod}_\beta \land \mathit{Sol}\land  r_{\P_\beta(\Phi)} \approx r^\mathrm{sol}_{s, q^\ast}\big],
\end{eqnarray*}
with $\mathit{Prod}_\beta$ expressing the dynamics of the product Markov chain.

\paragraph{Degree of observability terms}
Formulas of the form $\DDD_{\beta, i}(\Phi)$ are significantly more complicated than the previous terms. We will end up defining $\mathit{Eqn}_{\DDD_{\beta, i}(\Phi), s}$ as: 
\begin{align*}
\mathit{Eqn}_{\DDD_{\beta, i}(\Phi), s}& \defeq \exists r_{\P_\beta(\Phi)} \exists r_{(\P_\beta(\Phi))^{-1}}  \exists r_{\obs_i \Phi}. \\
& \big[\mathit{Eqn}_{\P_\beta(\Phi), s} \land \mathit{Eqn}_{\obs_i \Phi, s}\\
& \land r_{\P_\beta(\Phi)} \not\approx 0 \to [r_{(\P_\beta(\Phi))^{-1}} \times r_{\P_\beta(\Phi)} \approx 1 \\
& \land r_{\DDD_{\beta, i}(\Phi)} \approx r_{\obs_i \Phi} \times r_{(\P_\beta(\Phi))^{-1}}]\\
& \land r_{\P_\beta(\Phi)} \approx 0 \to r_{\DDD_{\beta, i}(\Phi)} \approx 1\big].    
\end{align*}

The idea is that we will compute two values: (i) the probability of a random path satisfying $\Phi$ (represented by the variable $r_{\P_\beta(\Phi)}$), and (ii) the probability of a random path satisfying $\Phi$ whilst \emph{not} being observationally equivalent (for agent $i$) to a $\lnot\Phi$-path (represented by the variable $r_{\obs_i \Phi}$). The degree of observability is then obtained by dividing (ii) by (i) (with some care to deal with the situation in which the denominator is 0). As we already know how to compute (i) from the inductive step for probabilistic terms shown before, we will only focus on computing (ii). Again, due to space constraints, the detailed construction of $\mathit{Eqn}_{\obs_i \Phi, s}$ can be found in the Appendix.
The idea is that we will, often internally in the real arithmetic formula, construct two series of automata, which we will then combine afterwards. 

First:
\begin{enumerate}[label=(A.\arabic*)]
    \item Compute a deterministic Streett automaton\footnote{A DSA over alphabet $A$ is a tuple $(Q, \delta, q^\ast, \mathsf{Acc})$, defined identically to a DRA. The difference now is that the automaton accepts $\mathbf{w} \in A^\omega$ iff for \emph{all} $(E, F) \in \mathsf{Acc}$, the automaton's run following $\mathbf{w}$ reaches all states in $E$ finitely often, \textbf{or} reaches some state in $F$ infinitely often. In other words, the acceptance condition is dual to a Rabin one.} recognizing $\Phi$ outside the real arithmetic formula.
    \item From that DSA, construct a DSA $A_\Phi$ over the alphabet $\ACT^\AG \times S$ outside the real arithmetic formula,
    such that $A_\Phi$ accepts precisely those $\mathbf{w}$ such that $s\mathbf{w} \in \paths_\GGG(s)$ and $s\mathbf{w} \models \Phi$. In other words, the DSA accepts the paths from $s$ that satisfy $\Phi$.
\end{enumerate}

\noindent
Second:
\begin{enumerate}[label=(B.\arabic*)]
    \item Compute an NBA recognizing $\lnot \Phi$ outside the real arithmetic formula.
    \item From that NBA, build an NBA $A^\mathsf{NBA}_{\obs_i \lnot \Phi}$ over alphabet $\ACT^\AG \times S$ inside the real arithmetic formula, such that $A^\mathsf{NBA}_{\obs_i \lnot \Phi}$ accepts precisely those $\mathbf{w}$ such that $s\mathbf{w} \in \paths_\MMM(s)$ and for which there exists $\pi \in \paths_\MMM(s)$ with $\pi \models \lnot \Phi$ and $\obs_i(s\mathbf{w}) = \obs_i(\pi)$. In other words, the NBA accepts the paths from $s$ that are (for agent $i$) observationally equivalent to paths satisfying $\lnot \Phi$.
    \item Determinize $A^\mathsf{NBA}_{\obs_i \lnot \Phi}$ into an equivalent DRA $A^\mathsf{DRA}_{\obs_i \lnot \Phi}$ using Safra's construction encoded into the real arithmetic formula.
    \item Negate $A^\mathsf{NBA}_{\obs_i \lnot \Phi}$ inside of the real arithmetic formula, obtaining a DSA $A^\mathsf{DSA}_{\obs_i \Phi}$ recognizing precisely those sequences $\mathbf{w} \in (\ACT^\AG \times S)^\omega$ such that if $s\mathbf{w} \in \paths_\MMM(s)$, then $s\mathbf{w}$ is not observationally equivalent (for $i$) to any path satisfying $\lnot \Phi$.
\end{enumerate}

We will then, finally, take the following steps:
\begin{enumerate}[label=(C.\arabic*)]
    \item Build the product DSA $A = A_\Phi \otimes A^\mathsf{DSA}_{\obs_i \Phi}$ inside of the real arithmetic formula, which recognizes precisely those $\mathbf{w}$ such that $s\mathbf{w} \in \paths_\MMM(s)$, $s\mathbf{w} \models \Phi$, and $s\mathbf{w}$ is not observationally equivalent to any path satisfying $\lnot \Phi$.
    \item Construct the product Markov chain $M_{\bm{\sigma}} \otimes A$ again inside the real arithmetic formula, and express the process of computing the probability of a random walk in it being accepting.
\end{enumerate}
\noindent
The idea is that we wish to end up with some deterministic automaton recognizing the set of $\Phi$-paths that are not observationally equivalent to a $\lnot\Phi$-path. This needs to be deterministic to be able to proceed with the product Markov chain construction. The A-construction builds a deterministic automaton that accepts precisely $\Phi$-paths, while the B-construction builds a deterministic automaton that accepts sequences that are not equivalent to $\lnot\Phi$-paths. By taking the product of both constructions, we end up with an automaton that recognizes the language we are after.

Note that many parts of the construction happen internally in the formula. This happens anytime a step relies on evaluating \opacpsl\ formulas, as such evaluation is only possible relative to a valuation, which we are quantifying over in real arithmetic, and therefore do not have access to. 

\subsection{Complexity}
\begin{theorem} \label{theo:complexity}
Model checking \opacpsl\ can be done in space triple exponential with respect to the \opacpsl\ sentence, and double exponential with respect to the partially observable multi-agent system.
\end{theorem}
\begin{proof}[Proof sketch]
The size of the real arithmetic formula is double exponential and the number of quantifiers is easily verified to be single exponential w.r.t. the sentence size. While the dependence of our real arithmetic formula's size on the \opacpsl\ sentence is the same as that of \psl, we do get another exponential blowup with respect to the size of the system. This extra blowup is caused by the larger number of quantifiers in our real arithmetic formula: we require at least one quantifier per Safra tree appearing in the B-construction in order to express the dynamics of the product Markov chain.

Since the validity of real arithmetic is decidable in space exponential w.r.t. the number of quantifiers and logarithmic w.r.t. the size of the quantifier-free part of the formula \cite{Renegar92,Basu14}, we therefore arrive at an overall space complexity that is triple exponential w.r.t. the \opacpsl\ sentence, and double exponential w.r.t. the system.
\end{proof}

\section{Conclusions and Future Work}
\label{sec:conclusion}
This paper provides a framework for expressing and reasoning about observability within MAS, along with the capability to quantify the degree of observability under specified strategies. The framework contributes to formal analysis and verification in multi-agent systems, especially for properties relating to information security, privacy, and trustworthiness. In particular, \opacpsl\ enables a rigorous analysis of agent observability and information transparency, allowing the assessment of how much information about system behaviours is available to different agents. This is crucial for identifying potential vulnerabilities and understanding the security implications of information exposure. 

In considering future directions, there are several areas that would be interesting to explore. First, the interconnections and synergies between \opacpsl\ and other logics, such as epistemic logics, would augment the framework's expressive capabilities. Another possible line of work involves extending \opacpsl\ to include additional aspects of multi-agent systems, such as hierarchical structures or more complex forms of actions. Adapting \opacpsl\ to navigate dynamic and evolving environments, where agents' strategies may undergo temporal transformations, presents another  area for investigation, as does investigating the application of the framework in the domains of AI safety and responsibility. Finally, incorporating game-theoretic approaches may allow balancing between utility and security.

\bibliographystyle{splncs03}
\bibliography{BIB-info-mas}

\newpage

\appendix
\section*{Appendix}
\label{sec:appendix}

\section{The Model Checking Procedure}
Take a partially observable multiagent system $\MMM = ((\AG, S, \ACT, T), \ASP, L, \{\obs_i\}_{i \in \AG})$, and an oPSL sentence $\varphi$. Note that since we are considering the semantics of oPSL with memoryless strategies, we do not need to consider history formulas interpreted on histories, but instead on individual states. For this reason, we will throughout the description of the model checking procedure refer to history formulas as \emph{state formulas}.

We will define first-order logic (FOL) formulas $\form_{\varphi, s}$ by induction over state formulas $\varphi$ 
and states $s$, written in the signature of real arithmetic. The construction will be such that
for all oPSL sentences $\varphi$ and states $s$,
$\form_{\varphi, s}$ holds in the theory of real arithmetic if and only if $s \models \varphi$. Using the decidability of the theory of real arithmetic, we then get a model checking algorithm.

Throughout this text we will often write $\top_{\mathit{cond}}$ for some metalogical condition $\mathit{cond}$, defined as $\top_\mathit{cond} \defeq \top$ if $\mathit{cond}$ is true, and $\top_\mathit{cond} \defeq \bot$ otherwise. We will denote the equality symbol inside of the real arithmetic formula by $\approx$ to avoid any confusion.

Also throughout this construction, we will often in our explanations ignore the presence of the valuation in the semantics, and write e.g. $s \models \varphi$ for non-sentences $\varphi$. We assume the valuation to still be there in the background, but it usually does not add much to the intuitions and explanations of the constructions, so we leave it out for brevity. Similarly, we will often write $\pi \models \Phi$ instead of $\pi, 0 \models \Phi$ for path formulas.

We will only speak of the complexity of the construction after we have fully described it.

\paragraph{Atoms and Booleans}
Let $\form_{p, s} \defeq \top_{p \in L(s)}$. We put $\form_{\varphi \land \psi, s} \defeq \form_{\varphi, s} \land \form_{\psi, s}$ and $\form_{\lnot \varphi, s} \defeq \lnot \form_{\varphi, s}$.

\paragraph{Existential quantifier}
Given the formula $\exists x.\varphi$, we introduce variables $r_{x, \theta, a}$, intuitively encoding the probability that the strategy $x$ performs action $a$ upon observing $\theta \in \Theta^\mathsf{s}$. We let $$\form_{\exists x.\varphi, s} \defeq \exists (r_{x, \theta, a})_{\theta \in \Theta^\mathsf{s}, a \in \ACT} [\mathit{Dist}_x \land \form_{\varphi, s}],$$
where
$$
\mathit{Dist}_x \defeq [\bigwedge_{\theta \in \Theta^\mathsf{s}, a \in \ACT} r_{x, \theta, a} \geqslant 0] \land [\bigwedge_{\theta \in \Theta^\mathsf{s}}\sum_{a \in \ACT} r_{x, \theta, a} \approx 1].
$$
The formula $\mathit{Dist}_x$ encodes that the variables $r_{x,\theta,a}$ give probability distributions for each $\theta$.

\paragraph{Full observability formula}
For the formula $\odot_k \Phi$, we will write a formula that expresses that it is \textbf{not} the case that $s \nmodels \odot_k \Phi$. The reason for this double negation is that we can relatively easily express $s \nmodels \odot_k \Phi$: this holds iff there exist paths $\pi, \pi' \in \paths_\MMM(s)$ such that $\pi \models \Phi$ and $\pi' \models \lnot \Phi$, but $\obs_k(\pi) = \obs_k(\pi')$. To check this, we will describe a rooted directed graph $\GGG_{\odot_k \Phi}$ inside of the real arithmetic formula, such that infinite walks through the graph (starting from the root) correspond precisely to pairs $\pi, \pi' \in \paths_\MMM(s)$. Furthermore, we will identify a set $F_{\odot_k\Phi}$ of vertices in the graph, such that infinite walks $(\pi, \pi')$ through the graph reach $F_{\odot_k \Phi}$ \emph{infinitely often} if and only if $\pi \models \Phi$, $\pi' \models \lnot \Phi$, and $\obs_k(\pi) = \obs_k(\pi')$. Given this, all we then need to express that $s \nmodels \odot_k \Phi$ is that there exists an infinite walk through the graph that hits $F_{\odot_k \Phi}$ infinitely often.

Note that $\Phi$ can be considered to be an LTL formula over atomic propositions $W = 2^{\mathsf{Max}(\Phi)}$, where $\mathsf{Max}(\Phi)$ is the set of maximal state subformulas of $\Phi$. So we can construct nondeterministic Büchi automata\footnote{An NBA over alphabet $A$ is a tuple $(Q, D, q^\ast, F)$ where $Q$ is a finite set of states, $D \subseteq Q \times A \times Q$ is a transition relation, $q^\ast \in Q$ is the initial state, and $F \subseteq Q$ is the set of accepting states. The automaton accepts an infinite sequence $\mathbf{w} \in A^\omega$ iff there exists some run of the automaton following $\mathbf{w}$, such that the run reaches some state in $F$ infinitely often.} $A_\Phi = (Q_\Phi, D_\Phi, q_\Phi^\ast, F_\Phi)$ and $A_{\lnot\Phi} = (Q_{\lnot\Phi}, D_{\lnot\Phi}, q_{\lnot\Phi}^\ast, F_{\lnot\Phi})$ over alphabet $W$, recognising those $\mathbf{w} \in W^\omega$ such that $\mathbf{w} \models_\mathsf{LTL} \Phi$ and $\mathbf{w} \models_\mathsf{LTL} \lnot \Phi$ respectively, where we write $\models_\mathsf{LTL}$ for the satisfaction relation of LTL.

Now we introduce some notation. We write $\Sigma_\xi = \ACT^\AG \cup \{\xi\}$, where $\xi$ is some distinguished joint action not in $\ACT^\AG$. We will as a convention put $\obs^\mathsf{a}_k(\xi) \defeq \xi$ for all agents' observation functions $\obs_k$. And similarly, as a convention we always put $T(t, \xi)(t') = 0$ for all states $t, t' \in S$.

Using these automata, we can describe our directed graph. We first do it outside of the real arithmetic formula so its construction is clear. Its vertex set is
$$V_{\odot_k \Phi} = \Sigma_\xi \times S \times Q_\Phi \times \Sigma_\xi \times S \times Q_{\lnot \Phi} \times 2,$$
where $2 = \{0, 1\}$. The root of the graph is the vertex $(\xi, s, q_\Phi^\ast, \xi, s, q_{\lnot\Phi}^\ast, 1)$. We have an edge from $(\alpha_1, s_1, q_1, \alpha'_1, s'_1, q'_1, b_1)$ to $(\alpha_2, s_2, q_2, \alpha'_2, s'_2, q'_2, b_2)$ if and only if:
\begin{enumerate}
    \item $D_\Phi(q_1, \{\varphi \in \mathsf{Max}(\Phi) \mid s_1 \models \varphi\}, q_2)$
    \item $D_{\lnot \Phi}(q'_1, \{\varphi \in \mathsf{Max}(\Phi) \mid s'_1 \models \varphi\}, q'_2)$
    \item $T(s_1, \alpha_2)(s_2) > 0$
    \item $T(s'_1, \alpha'_2)(s'_2) > 0$
    \item If $\mathsf{obs}_k(\alpha_2, s_2) = \mathsf{obs}_k(\alpha'_2, s'_2)$, then $b_2 = b_1$
    \item If $\mathsf{obs}_k(\alpha_2, s_2) \neq \mathsf{obs}_k(\alpha'_2, s'_2)$, then $b_2 = 0$
\end{enumerate}

Intuitively, walks through the graph consist of two paths in $\MMM$. For the first path, we keep track of the corresponding state of $A_\Phi$; for the second, we keep track of $A_{\lnot\Phi}$. At the same time, we keep track of a bit $b \in 2$ that expresses whether the two paths have been observationally equivalent for agent $k$ so far.

We consider the set $F_{\odot_k\Phi}$ of distinguished vertices to be $F_{\odot_k \Phi} = \Sigma_\xi \times S \times F_\Phi \times \Sigma_\xi \times S \times F_{\lnot\Phi} \times \{1\}$. It is immediate from our construction and the defining property of $A_\Phi$ and $A_{\lnot\Phi}$ that there exist infinite walks through the graph from the root that hit this set infinitely often, if and only if $s \nmodels \odot_k \Phi$.

We note that since this graph has finitely many vertices, we have that there exist infinite walks with the property mentioned above, if and only if there exists some vertex $v \in F_{\odot_k \Phi}$ that is reachable from the root, and that has a non-empty path back to itself. It is precisely this that we will express in the real arithmetic formula.

For $\Psi \in \{\Phi, \lnot\Phi\}$, $s \in S$ and $q, q' \in Q_\Psi$, we write a formula $\mathit{Tr}_{\Psi, q, q'}^s$ expressing that $A_\Psi$ transitions from $q$ to $q'$ upon receiving the valuation of $s$.
$$
\mathit{Tr}_{\Psi, q, q'}^s \defeq \bigvee_{\substack{X \in W \\ D_\Psi(q, X, q')}}\big[\bigwedge_{\varphi \in X}\form_{\varphi, s} \land \bigwedge_{\varphi \in \mathsf{Max}(\Phi) \setminus X} \lnot \form_{\varphi, s}\big].
$$

Next, we define a formula $\mathit{Edge}_{\alpha_1, s_1, q_1, \alpha'_1, s'_1, q'_1, b_1}^{\alpha_2, s_2, q_2, \alpha'_2, s'_2, q'_2, b_2}$, stating that there is an edge in the directed graph $\mathcal{G}_{\odot_k\Phi}$ from the vertex in the subscript to the vertex in the superscript.

\begin{align*}
   \mathit{Edge}_{\alpha_1, s_1, q_1,\alpha'_1, s'_1, q'_1, b_1}^{\alpha_2,s_2, q_2, \alpha'_2, s'_2, q'_2, b_2} \defeq&  \mathit{Tr}_{\Phi, q_1, q_2}^{s_1} \land \mathit{Tr}_{\lnot \Phi, q'_1, q'_2}^{s'_1}\\
   &\land \top_{T(s_1, \alpha_2)(s_2) > 0}\\
   &\land \top_{T(s'_1, \alpha'_2)(s'_2) > 0}\\
    &\land \top_{\obs_k(\alpha_2, s_2) = \obs_k(\alpha'_2, s'_2) \implies b_1 = b_2}\\
    &\land \top_{\obs_k(\alpha_2, s_2) \neq \obs_k(\alpha'_2, s'_2) \implies b_2 = 0}
\end{align*}

Next, we need to express that there exist some vertex $v \in F_{\odot_k \Phi}$ that is reachable from the root, and that has a non-empty path back to itself. So we need to be able to express \emph{reachability}. To do this efficiently, we use the method well-known from second-order logic: a vertex $v$ reaches $w$ iff for all edge-closed sets $C$ of vertices, we have that $v \in C$ implies $w \in C$.\begin{nm}This next footnote should only be included after acceptance, as it does somewhat break anonymity.\footnote{We thank the authors of the original PSL paper for pointing out this efficient construction to us.}\end{nm}

To be able to quantify over sets of vertices, we introduce variables $r^\mathrm{mark}_v$ for each vertex $v$, with the intuition being that $r^\mathrm{mark}_v \neq 0$ denotes that $v$ is part of the subset $X$ of vertices. The following formula $\mathit{Cl}$ expresses that the set defined by these variables is closed under the edge relation.
\begin{align*}
\mathit{Cl} \defeq \bigwedge_{v, w \in V_{\odot_k \Phi}} [r^\mathrm{mark}_v \not\approx 0 \land \mathit{Edge}_v^w] \to r^\mathrm{mark}_w \not\approx 0
\end{align*}

Using this, we define reachability.
\begin{align*}
\mathit{Reach}_{\odot_k \Phi, v}^w \defeq& 
\forall (r^\mathrm{mark}_{u})_{u \in V_{\odot_k \Phi}}[\mathit{Cl} \land r^\mathrm{mark}_v \not\approx 0] \to r^\mathrm{mark}_w \not\approx 0
\end{align*}

We can now finally define $\form_{\odot_k\Phi, s}$ by stating that we cannot reach vertices in $F_{\odot_k \Phi}$ from the root that go back to themselves with non-empty paths.
\begin{align*}
\form_{\odot_k\Phi, s} \defeq & \; \lnot \bigvee_{v \in F_{\odot_k\Phi}} \big[ \mathit{Reach}_{\odot_k\Phi, (\xi, s, q^\ast_\Phi, \xi, s, q^\ast_{\lnot\Phi}, 1)}^{v}
\land \bigvee_{v' \in V_{\odot_k\Phi}} [\mathit{Edge}_{v}^{v'} \land \mathit{Reach}_{\odot_k\Phi, v'}^{v}]\big]
\end{align*}

\paragraph{Inequalities}
Given the formula $\tau_1 \leqslant \tau_2$, we introduce variables $r_{\tau}$ for all arithmetic subterms $\tau$ appearing in $\tau_1$ and $\tau_2$ (and we denote the set of all such subterms by $\mathit{Sub}(\tau_1, \tau_2)$), which will intuitively hold the values the terms $\tau$ will take at the state $s$ under consideration, and given the strategies assigned to the strategic variables $x$. We let
$$
\form_{\tau_1 \leqslant \tau_2, s} \defeq \exists (r_{\tau})_{\tau \in \mathit{Sub}(\tau_1, \tau_2)}[\mathit{Eqn}_{\tau_1, s} \land \mathit{Eqn}_{\tau_2, s} \land r_{\tau_1} \leqslant r_{\tau_2}],
$$
wherein the formula $\mathit{Eqn}_{\tau_i, s}$ encodes that the variable $r_{\tau_i}$ indeed holds the value of $\tau_i$ in $s$ under the current valuation. We will now define these formulas.

\paragraph{Arithmetic terms}
We let:
$$\mathit{Eqn}_{c, s} \defeq [r_c \approx c],$$ 
$$\mathit{Eqn}_{\tau^{-1}, s} \defeq \mathit{Eqn}_{\tau, s} \land [r_{\tau} \times r_{\tau^{-1}} \approx 1],$$ 
$$\mathit{Eqn}_{\tau_1 + \tau_2, s} \defeq \mathit{Eqn}_{\tau_1, s} \land \mathit{Eqn}_{\tau_2, s} \land [r_{\tau_1 + \tau_2} \approx r_{\tau_1} + r_{\tau_2}],$$
$$\mathit{Eqn}_{\tau_1 \times \tau_2, s} \defeq \mathit{Eqn}_{\tau_1, s} \land \mathit{Eqn}_{\tau_2, s} \land [r_{\tau_1 \times \tau_2} \approx r_{\tau_1} \times r_{\tau_2}],$$


\paragraph{Probabilistic terms}
Consider the term $\P_\beta (\Phi)$. The idea here is that we will `internalize' the description of the model checking procedure of PCTL* inside of our real arithmetic formula.

Considering $\Phi$ to be an LTL formula over $W = 2^{\mathsf{Max}(\Phi)}$ like in the construction  for the full observability formula, we construct a deterministic Rabin automaton\footnote{A DRA over alphabet $A$ is a tuple $(Q, \delta, q^\ast, \mathsf{Acc})$, where $Q$ is a finite set of states, $\func{\delta, Q \times A, Q}$ is the transition function, $q^\ast \in Q$ is the initial state, and $\mathsf{Acc} \subseteq 2^Q \times 2^Q$ is the \emph{Rabin acceptance condition}: a set of tuples $(E, F)$ of subsets $E, F \subseteq Q$. The automaton accepts an infinite sequence $\mathbf{w} \in A^\omega$ iff there exists some $(E, F) \in \mathsf{Acc}$ such that on the automaton's run following $\mathbf{w}$, it reaches all states in $E$ \emph{finitely} often, and reaches some state in $F$ \emph{infinitely} often.} $A_\Phi = (Q, q^\ast, \delta, \mathsf{Acc})$ over alphabet $W$ accepting those $\mathbf{w} \in W^\omega$ such that $\mathbf{w} \models_\mathsf{LTL} \Phi$.

Before giving the real arithmetic construction, we will again describe outside of it what we aim to do. Recall that a (memoryless) strategy profile $\bm{\sigma}$ (as e.g. given by a binding $\beta$ and valuation $\nu$) turns our model $\MMM$ into a Markov chain $M_{\bm{\sigma}} = (S, T_{\bm{\sigma}})$, with $\func{T_{\bm{\sigma}}, S, \dist(S)}$ defined by putting $$T_{\bm{\sigma}}(s_1)(s_2) = \sum_{\alpha \in \ACT^\AG} T(s_1, \alpha)(s_2) \prod_{i \in \AG} \sigma_i(\obs_i(s_1))(\alpha_i).$$
We wish to compute the probability that a random walk through this Markov chain starting at $s$ satisfies $\Phi$. To do so, we will employ the product Markov chain construction as is used in PCTL* model checking.\footnote{Note that the following construction can only be defined because our automaton is deterministic, which explains why we did not just construct an NBA.} The product Markov chain $M_{\bm{\sigma}} \otimes A$ will have the same dynamics as $M_{\bm{\sigma}}$, whilst also following the automaton $A$. More precisely, we are defining the Markov chain 
$M_{\bm{\sigma}} \otimes A = (S \times Q, T_{\bm{\sigma}, A}$), with $T_{\bm{\sigma}, A}$ being given as
$$
T_{\bm{\sigma}, A}(s_1, q_1)(s_2, q_2) =
\begin{cases}
T_{\bm{\sigma}}(s_1)(s_2) &\text{if}\ \delta(q_1, \{\varphi \in \mathsf{Max}(\Phi) \mid s_1 \models \varphi\}) = q_2\\
0 &\text{otherwise}
\end{cases}
$$
It follows from this construction and the defining property of the DRA $A$, that the probability we are after is precisely the probability that a random walk through $M_{\bm{\sigma}} \otimes A$ from $(s, q^\ast)$ is `accepted' by $A$, in the sense that the walk's sequence of automaton states $q_0 q_1 \cdots \in Q^\omega$ satisfies the Rabin acceptance condition $\mathsf{Acc}$.

Now, one of the fundamental properties of Markov chains that is used in PCTL* model checking states that for all states in a finite-state Markov chain, a random walk starting there will with probability 1 (i) eventually reach some terminal strongly connected component (tSCC)\footnote{A set $C$ of states of a Markov chain is called a strongly connected component, if it is a strongly connected component of the Markov chain's underlying directed graph (i.e. the one obtained by placing edges between states that have non-zero transition probability). Furthermore, $C$ is said to be terminal if it is terminal in the directed graph (where a strongly connected component of a directed graph is terminal if there exists no vertex outside the component that is reachable from within).}, (ii) stay in that tSCC forever, and (iii) visit each state in that tSCC infinitely often. 

Applying this property to $M_{\bm{\sigma}} \otimes A$, we see that the probability of a random walk from $(s, q^\ast)$ being accepted by $A$ is equal to the probability of a random walk eventually reaching some tSCC $C$, such that $C \cap (S \times E) = \emptyset$ and $C \cap (S \times F) \neq \emptyset$ for some $(E, F) \in \mathsf{Acc}$. We refer to such tSCCs as \emph{accepting tSCCs}.

So we have reduced the problem to computing a \emph{reachability} probability, which can be done by solving a system of linear equations. Using real-valued variables $p_{t, q}$ for $(t, q) \in S \times Q$, which intuitively will denote the probability of reaching an accepting tSCC from $(t,q)$, we solve the following system:
$$
\begin{cases}
p_{t, q} = 1 &\text{if $(t,q)$ lies in some accepting tSCC}\\
p_{t, q} = 0 &\text{if $(t,q)$ cannot reach any accepting tSCC}\\
p_{t,q} = \sum_{(t',q') \in S \times Q}T_{\bm{\sigma}, A}(t,q)(t', q') \times p_{t', q'}&\text{otherwise}
\end{cases}
$$
It can be verified (though it goes beyond our interests here) that this system always has a unique solution. Upon solving the system, the probability we are after is the value of $p_{s, q^\ast}$.

Having described generally what we intend to do, we will now describe how to encode this process (efficiently) in real arithmetic.

First, for $t \in S$ and $q, q' \in Q$, we write a formula $\mathit{Tr}_{q, q'}^t$ expressing that $A$ indeed transitions from $q$ to $q'$ upon receiving $s$.
$$
\mathit{Tr}_{q, q'}^t \defeq \bigvee_{\substack{X \in W \\ \delta(q, X) = q'}}\big[\bigwedge_{\varphi \in X}\form_{\varphi, t} \land \bigwedge_{\varphi \in \mathsf{Max}(\Phi) \setminus X} \lnot \form_{\varphi, t}\big].
$$

Introduce variables $r^{s_2, q_2}_{s_1, q_1}$ for $s_1, s_2 \in S$ and $q_1, q_2 \in Q$, which are meant to express the probability of transitioning from $(s_1, q_1)$ to $(s_2, q_2)$ in the product Markov chain we obtain from the current valuation (given by the variables $r_{x, \theta, a}$) and binding $\beta$. The formula $\mathit{Prod}_\beta$ expresses this:
\begin{align*}
\mathit{Prod}_\beta  \defeq \bigwedge_{\substack{s_1, q_1\\s_2, q_2}}& \big[\mathit{Tr}_{q_1, q_2}^{s_1} \to r^{s_2, q_2}_{s_1, q_1} \approx \sum\limits_{\alpha \in \mathit{Ac}^\mathit{Ag}}T(s_1, \alpha)(s_2)\prod\limits_{i \in \mathit{Ag}} r_{\beta(i), \obs_i(s_1), \alpha_i}\big]\\
&\land [\lnot \mathit{Tr}_{q_1, q_2}^{s_1} \to r^{s_2, q_2}_{s_1, q_1} \approx 0].
\end{align*}

We write for each $v, w \in S \times Q$ a formula $\mathit{Reach}_v^{w}$ expressing $v$ reaches $w$ in the underlying directed graph of the product chain. Again, like for the full observability formula, this makes use of variables $r^\mathrm{mark}_v$ and a formula $\mathit{Cl}$ expressing that we have a set of vertices closed under the edge relation.
\begin{align*}
\mathit{Cl} \defeq \bigwedge_{v, w \in S \times Q} [r^\mathrm{mark}_v \not\approx 0 \land r_v^w > 0] \to r^\mathrm{mark}_w \not\approx 0
\end{align*}

\begin{align*}
\mathit{Reach}_v^w \defeq \forall (r^\mathrm{mark}_{u})_{u \in S \times Q}[\mathit{Cl} \land r^\mathrm{mark}_v \not\approx 0] \to r^\mathrm{mark}_w \not\approx 0
\end{align*}

We now move on to encoding the process of solving the system of equations. 
We introduce variables $r^\mathrm{sol}_{t, q}$ for each $(t, q)$ which will contain the probability $p_{t,q}$ of reaching some accepting tSCC from $(t, q)$. The formula $\mathit{Goal}_v$ efficiently expresses that $v$ is part of an accepting tSCC, without having to explicitly consider sets of vertices at any point.
\begin{align*}
\mathit{Goal}_v \defeq [\bigwedge_{w \in S \times Q} \mathit{Reach}_v^w \to \mathit{Reach}_w^v]\land \bigvee_{(E,F) \in \mathsf{Acc}}[\bigvee_{w \in S \times F} \mathit{Reach}^{w}_{v}
\land \bigwedge_{w \in S \times E} \lnot\mathit{Reach}_v^w]
\end{align*}
To see that this encoding indeed works, note first that a vertex $v$ of any directed graph graph lies in a tSCC if and only if for every $w$ that $v$ reaches, we have that $w$ also reaches $v$. This should be easily verifiable. That explains the first part of the formula.

Then, we note second that the tSCC $v$ is a part of in the product Markov chain is accepting if and only if there is some $(E, F) \in \mathsf{Acc}$ such that $v$ does not reach any vertex $w \in S \times E$, and reaches some vertex $v \in S \times F$. To verify this, note that the tSCC $v$ is a part of has non-empty intersection with a set $X$ if and only if $v$ reaches some state in $X$. Left-to-right holds since the tSCC is an SCC: every pair of vertices in it can reach each other. Right-to-left holds since any vertex $w$ that $v$ reaches must also reach $v$, and so since the component is terminal, we must have that $w$ is a part of it. 

We can now write the formula $\mathit{Sol}$, expressing the solving of the system of equations.
\begin{align*}
\mathit{Sol} \defeq \bigwedge_{v \in S \times Q} &\mathit{Goal}_{v} \to r^\mathrm{sol}_{v} \approx 1\\
&\land [\lnot\bigvee_{w \in S \times Q} \mathit{Goal}_{w} \land \mathit{Reach}_{v}^{w}] \to r^\mathrm{sol}_{v} \approx 0\\
&\land [\lnot \mathit{Goal}_v \land \bigvee_{w \in S \times Q} \mathit{Goal}_{w} \land \mathit{Reach}_{v}^{w}] \to r^\mathrm{sol}_{v} \approx \sum_{w\in S \times Q}[r_{v}^{w} \times r^\mathrm{sol}_{w}]\\
\end{align*}
Putting all this together, we can finally write out $\mathit{Eqn}_{\P_\beta (\Phi), s}$:
\begin{align*}
\mathit{Eqn}_{\P_\beta (\Phi), s}
    \defeq \exists (r_{v}^{w}, r^\mathrm{sol}_{v})_{v, w \in S \times Q} \big[ \mathit{Prod}_\beta \land \mathit{Sol}\land  r_{\P_\beta(\Phi)} \approx r^\mathrm{sol}_{s, q^\ast}\big]
\end{align*}

\paragraph{Degree of observability terms}
Consider the term $\DDD_{\beta, i}(\Phi)$. This will be quite a bit more complicated than the previous terms.

We will end up defining $\mathit{Eqn}_{\DDD_{\beta, i}(\Phi), s}$ as 
\begin{align*}
\mathit{Eqn}_{\DDD_{\beta, i}(\Phi), s} \defeq& \exists r_{\P_\beta(\Phi)} \exists r_{(\P_\beta(\Phi))^{-1}}  \exists r_{\obs_i \Phi}.\big[\mathit{Eqn}_{\P_\beta(\Phi), s} \land \mathit{Eqn}_{\obs_i \Phi, s}\\
& \land r_{\P_\beta(\Phi)} \not\approx 0 \to [r_{(\P_\beta(\Phi))^{-1}} \times r_{\P_\beta(\Phi)} \approx 1 \\
& \land r_{\DDD_{\beta, i}(\Phi)} \approx r_{\obs_i \Phi} \times r_{(\P_\beta(\Phi))^{-1}}]\\
& \land r_{\P_\beta(\Phi)} \approx 0 \to r_{\DDD_{\beta, i}(\Phi)} \approx 1\big].    
\end{align*}
The idea is that we will compute two values: (i) the probability of a random path satisfying $\Phi$, and (ii) the probability of a random path satisfying $\Phi$ whilst \emph{not} being observationally equivalent (for agent $i$) to a $\lnot\Phi$-path. The degree of observability is then obtained by dividing (ii) by (i) (with some care to deal with the situation in which the denominator is 0). As we already know how to compute (i) from the inductive step for probabilistic terms shown before, we will only focus on computing (ii), which in the real arithmetic formula we will store in the variable $r_{\obs_i \Phi}$. So all we need to do is explain the construction of $\mathit{Eqn}_{\obs_i \Phi, s}$.

We will again begin by describing a procedure for computing (ii) in general, and then encode this in real arithmetic.

We will construct two series of automata, which we will then combine afterwards. First:
\begin{enumerate}[label=(A.\arabic*)]
    \item Compute a deterministic Streett automaton\footnote{A DSA over alphabet $A$ is a tuple $(Q, \delta, q^\ast, \mathsf{Acc})$, defined identically to a DRA. The difference now is that the automaton accepts $\mathbf{w} \in A^\omega$ iff for \emph{all} $(E, F) \in \mathsf{Acc}$, the automaton's run following $\mathbf{w}$ reaches all states in $E$ finitely often, \textbf{or} reaches some state in $F$ infinitely often. In other words, the acceptance condition is dual to a Rabin one.} recognising $\Phi$ outside of the real arithmetic formula.
    \item From that DSA, build a DSA $A_\Phi$ over the alphabet $\ACT^\AG \times S$ \emph{inside} the real arithmetic formula,
    such that $A_\Phi$ accepts precisely those $\mathbf{w}$ such that $s\mathbf{w} \in \paths_\MMM(s)$ and $s\mathbf{w} \models \Phi$. In other words, the DSA accepts the paths from $s$ that satisfy $\Phi$.
\end{enumerate}

Second:
\begin{enumerate}[label=(B.\arabic*)]
    \item Compute an NBA recognising $\lnot \Phi$ outside the real arithmetic formula.
    \item From that NBA, build an NBA $A^\mathsf{NBA}_{\obs_i \lnot \Phi}$ over alphabet $\ACT^\AG \times S$ inside the real arithmetic formula, such that $A^\mathsf{NBA}_{\obs_i \lnot \Phi}$ accepts precisely those $\mathbf{w}$ such that $s\mathbf{w} \in \paths_\MMM(s)$ and for which there exists $\pi \in \paths_\MMM(s)$ with $\pi \models \lnot \Phi$ and $\obs_i(s\mathbf{w}) = \obs_i(\pi)$. In other words, the NBA accepts the paths from $s$ that are (for agent $i$) observationally equivalent to paths satisfying $\lnot \Phi$.
    \item Determinize $A^\mathsf{NBA}_{\obs_i \lnot \Phi}$ into an equivalent DRA $A^\mathsf{DRA}_{\obs_i \lnot \Phi}$ using Safra's construction encoded into the real arithmetic formula.
    \item Negate $A^\mathsf{NBA}_{\obs_i \lnot \Phi}$ inside of the real arithmetic formula, obtaining a DSA $A^\mathsf{DSA}_{\obs_i \Phi}$ recognising precisely those sequences $\mathbf{w} \in (\ACT^\AG \times S)^\omega$ such that if $s\mathbf{w} \in \paths_\MMM(s)$, then $s\mathbf{w}$ is not observationally equivalent (for $i$) to any path satisfying $\lnot \Phi$.
\end{enumerate}

We will then, finally, take the following steps:
\begin{enumerate}[label=(C.\arabic*)]
    \item Build the product DSA $A = A_\Phi \otimes A^\mathsf{DSA}_{\obs_i \Phi}$ inside of the real arithmetic formula, which recognizes precisely those $\mathbf{w}$ such that $s\mathbf{w} \in \paths_\MMM(s)$, $s\mathbf{w} \models \Phi$, and $s\mathbf{w}$ is not observationally equivalent to any path satisfying $\lnot \Phi$.
    \item Construct the product Markov chain $M_{\bm{\sigma}} \otimes A$ again inside the real arithmetic formula, and compute the probability of a random walk in it being accepting.
\end{enumerate}

The idea is that we wish to end up with some deterministic automaton recognising the set of $\Phi$-paths that are not obs. equiv. to a $\lnot\Phi$-path. This needs to be deterministic to be able to proceed with the product Markov chain construction. The A-construction builds a deterministic automaton that accepts precisely $\Phi$-path, while the $B$-construction builds a deterministic automaton that accepts sequences that are not equivalent to $\lnot\Phi$-paths. By taking the product of both constructions, we end up with an automaton that recognizes the language we are after.

The A-construction is straightforward, but describing the B-construction requires more steps. In order to accept paths that are not observationally equivalent to $\lnot\Phi$-paths, we make use of a nondeterministic automata that accepts paths that \emph{are} observationally equivalent to one: intuitively, such an automaton is simply able to guess the other path. So instead of constructing a deterministic automaton from $\lnot\Phi$, we construct a nondeterministic Buchi automaton. Then, using Safra's construction, we can translate this into a deterministic Rabin automaton. We can very simply construct the complement of a DRA as a DSA. This is also why we do not contruct Rabin automata in the A-construction, but instead go straight for a Streett automaton.

Let us begin by describing the A-construction (note we are not yet presenting how the construction is done inside the real arithmetic formula), as it is the simplest. From $\Phi$, we construct a DSA $A_\Phi^0 = (Q, \delta, q^\ast, \mathsf{Acc})$ over $W = 2^{\mathsf{Max}(\Phi)}$ that recognizes $\Phi$. From this DSA, we construct the DSA $A_\Phi = (Q_\Phi, \delta_\Phi, (s, q^\ast), \mathsf{Acc}_\Phi)$ as follows. The state space is $Q_\Phi \defeq (S \times Q) \cup \{v_\mathsf{err}\}$, i.e. $S \times Q$ with a separate error state $v_\mathsf{err}$. We define $\delta_\Phi$ as
$$
\delta_\Phi((t, q), (\alpha, t')) \defeq
\begin{cases}
(t', \delta(q, \{\varphi \in \mathsf{Max}(\Phi) \mid t \models \varphi\})) &\text{if}\ T(t, \alpha)(t') > 0\\
v_\mathsf{err} &\text{otherwise},
\end{cases}
$$
and $\delta_\Phi(v_\mathsf{err}, (\alpha, t)) \defeq v_\mathsf{err}$. The Streett acceptance condition is $\mathsf{Acc}_\Phi \defeq \{(S \times E, S \times F) \mid (E, F) \in \mathsf{Acc}\} \cup \{(\{v_\mathsf{err}\}, S \times Q)\}$. It is not difficult to see that this construction is correct: the acceptance condition forces only those runs to be accepted which never reach the error state (so sequences that do not correspond to paths will never be accepted.)

Now we describe the B-construction's process (not yet in the real arithmetic formula). Let $A_{\lnot\Phi} = (Q_{\lnot\Phi}, D_{\lnot\Phi}, q^\ast_{\lnot\Phi}, F_{\lnot\Phi})$ be an NBA over $2^{\mathsf{Max}(\lnot\Phi)}$ that recognizes $\lnot\Phi$. We construct the NBA $\A^\mathsf{NBA}_{\obs_i\lnot\Phi} = (S \times S \times Q_{\lnot\Phi}, D^\mathsf{NBA}_{\obs_i\lnot\Phi}, (s, s, q^\ast_{\lnot\Phi}), F^\mathsf{NBA}_{\obs_i\lnot\Phi})$ over $\ACT^\AG \times S$ as follows.

The state space is $Q^\mathsf{NBA}_{\obs_i \lnot \Phi} = S \times S \times Q_{\lnot\Phi}$: states consist of two states of $\MMM$ and one state of the original NBA. Intuitively, the first $\MMM$-state will keep track of the path we are given, while the second $\MMM$-state will be that of the observationally equivalent path we are guessing.

The transition relation $D^\mathsf{NBA}_{\obs_i\lnot\Phi}$ is defined by having a transition from $(s_1, s'_1, q_1)$ to $(s_2, s'_2, q_2)$ under symbol $(\alpha, t)$ if and only if:
\begin{enumerate}
    \item $t = s_2$
    \item $T(s_1, \alpha)(s_2) > 0$
    \item There is $\alpha' \in \ACT^\AG$ such that $T(s'_1, \alpha')(s'_2) > 0$, and $\obs_i(\alpha, s_2) = \obs_i(\alpha', s'_2)$.
    \item $D_{\lnot\Phi}(q_1, \{\varphi \in \mathsf{Max}(\Phi) \mid s'_1 \models \varphi\}, q_2)$
\end{enumerate}
So transitions follow the symbol on the first $\MMM$-state, and guess some observationally equivalent transition from the second $\MMM$-state, while also transitioning the original NBA for the second state.

The set $F^\mathsf{NBA}_{\obs_i\lnot\Phi}$ of accepting states is defined as $F^\mathsf{NBA}_{\obs_i\lnot\Phi} \defeq S \times S \times F_{\lnot\Phi}$. It is not difficult to see that the resulting NBA is what we are after: it accepts precisely paths that are observationally equivalent to $\lnot\Phi$-paths.

Next, we apply Safra's well-known determinization construction\footnote{The precise construction we are using is the one as presented in: Roggenbach, M. (2002). Determinization of Büchi-Automata. In: Grädel, E., Thomas, W., Wilke, T. (eds) Automata Logics, and Infinite Games. Lecture Notes in Computer Science, vol 2500. Springer, Berlin, Heidelberg. \url{https://doi.org/10.1007/3-540-36387-4_3}} to construct a DRA $A^\mathsf{DRA}_{\obs_i\lnot\Phi}$ that is equivalent to our NBA. Since we will later need to encode this in our real arithmetic formula, we work out the construction here.

The construction is based on the notion of a \emph{Safra tree}, which we will now present relative to our NBA $A^\mathsf{NBA}_{\obs_i\lnot\Phi}$. A Safra tree  is a finite ordered tree $\mathcal{T}$, with its nodes coming from the vocabulary $V = \{1, 2, \dots, 2(\lvert S \rvert^2 \times \lvert Q_{\lnot\Phi} \rvert)\}$ (i.e. there is twice as many symbols as the amount of states in our NBA). By ordered tree, we mean a tree in which the children of every node carry some linear order $<$. We will often speak of children of a node being earlier/later than another to mean that the first, respectively, precedes ($<$) or succeeds $(>)$ the other in this linear order. Similarly, the last child of a node is its child that succeeds all other children of that node.

Nodes $v$ of a Safra tree $\mathcal{T}$ are labelled with a \emph{macrostate} $K_v^\mathcal{T}$, which is a non-empty subset $K_v^\mathcal{T} \subseteq S \times S \times Q_{\lnot\Phi}$. Nodes can also be marked with an additional symbol `$!$'.

Safra trees are required to satisfy two conditions, which also as a consequence mean that the number of Safra trees is exponential in the number of states of our NBA. For a Safra tree $\mathcal{T}$ we require that for all nodes $v$:
\begin{enumerate}
    \item There is at least one $(t, t',q) \in K^\mathcal{T}_v$ such that $v$ has no child $w$ with $(t, t', q) \in K_w^\mathcal{T}$
    \item For all distinct children $w \neq u$ of $v$, it holds that $K_w^\mathcal{T} \cap K_u^\mathcal{T} = \emptyset$.
\end{enumerate}

The DRA $A^\mathsf{DRA}_{\obs_i \lnot \Phi} = (Q^\mathsf{DRA}, \delta^\mathsf{DRA}, q_\ast^\mathsf{DRA}, \mathsf{Acc}^\mathsf{DRA})$ is defined as follows. Its state space $Q^\mathsf{DRA}$ is the set of Safra trees. The initial state $q_\ast^\mathsf{DRA}$ is the Safra tree with the single node $1$, labelled with macrostate $\{(s, s, q^\ast)\}$ (i.e. the singleton containing the initial state of our NBA). 

The transition function is defined as follows. Given a Safra tree $\mathcal{T}$ with nodes $N \subseteq V$, and input symbol $(\alpha, t) \in \ACT^\AG \times S$, the transition $\delta^\mathsf{DRA}(\mathcal{T}, (\alpha, t))$ produces the Safra tree computed as follows:
\begin{enumerate}
    \item The mark `$!$' is removed from all nodes in $\mathcal{T}$ that contain it
    \item For every node $v \in N$ with macrostate $K$ such that $K \cap F^\mathsf{NBA}_{\obs_i\lnot\Phi} \neq \emptyset$, a new node $w \in V \setminus N$ is added, and is made the new last child of $v$. The macrostate of $w$ is set to be $K \cap F^\mathsf{NBA}_{\obs_i\lnot\Phi}$.
    \item For every node $v$, its macrostate $K$ is replaced by the new macrostate $$\{(t_2, t'_2, q_2) \in S \times S \times Q_{\lnot\Phi} \mid D^\mathsf{NBA}_{\obs_i\lnot\Phi}((t_1, t'_1, q_1), (\alpha, t), (t_2, t'_2, q_2))\ \text{for some}\ (t_1, t'_1, q_1) \in K\}$$ 
    \item For every node $v$ with macrostate $K$, and all states $(t, t', q) \in K$ such that $v$ has an earlier sibling whose macrostate includes $(t, t', q)$, we remove $(t, t', q)$ from the macrostate of $v$.
    \item We remove all nodes with empty macrostates.
    \item For every node $v$ such that its macrostate is equal to the union of the macrostates of its children, we remove all nodes descended from $v$, and place the mark `$!$' on $v$.
\end{enumerate}

The Rabin acceptance condition $\mathsf{Acc}^\mathsf{DRA}$ is defined to be $\mathsf{Acc}^\mathsf{DRA} \defeq \{(E_v, F_v) \mid v \in V\}$, with $E_v$ being the set of those Safra trees that do not contain a node $v$, and $F_v$ being the set of all Safra trees that do contain a node $v$ which is additionally marked with `$!$'.

This finishes the description of the DRA. The DSA recognising its complement is obtained by just letting $A^\mathsf{DSA}_{\obs_i \Phi} = (Q_{\obs_i \Phi}, \delta_{\obs_i \Phi}, q^\ast_{\obs_i \Phi}, \mathsf{Acc}_{\obs_i \Phi})$ be defined identically to $A^\mathsf{DRA}_{\obs_i \lnot \Phi}$, but with the \emph{Streett} acceptance condition $\mathsf{Acc}_{\obs_i \Phi} \defeq \{(Q^\mathsf{DRA} \setminus E, Q^\mathsf{DRA} \setminus F) \mid (E, F) \in \mathsf{Acc}^\mathsf{DRA}\}$, i.e. the we take complements of all the sets in the condition.

Now, having completed the general descriptions of the A- and B-constructions, we can finally describe the C-construction. We will only describe the first step, since the second is virtually identical to our procedure for checking probabilistic terms presented before (with the exception that we now deal with a Streett condition instead of Rabin). The product DSA $A$ is defined by taking the products of the DSAs $A_\Phi$ and $A^\mathsf{DSA}_{\obs_i \Phi}$. The state space, transition function, and initial state are defined in the natural way. The acceptance condition $\mathsf{Acc}_\mathsf{prod}$ consists of all pairs $(E \times Q_{\obs_i \Phi}, F \times Q_{\obs_i \Phi})$ for $(E, F) \in \mathsf{Acc}_\Phi$, and all pairs $(Q_\Phi \times E, Q_\Phi \times F)$ for $(E, F) \in \mathsf{Acc}_{\obs_i\Phi}$.

We can now finally present the encoding of all the A, B and C-constructions in real arithmetic. We start with the A-construction. We begin by specifying the transition function of $A_\Phi$. The formula $\mathit{Tr}_{A_\Phi, x, y}^{\alpha,t}$ expresses that $x \in Q_\Phi$ transitions to $y \in Q_\Phi$ under symbol $(\alpha, t')$.
\begin{align*}
    \mathit{Tr}_{A_\Phi, x, y}^{\alpha, t} \defeq 
    \begin{cases}
    \top_{T(t, \alpha)(t') > 0} \land \bigvee\limits_{\substack{X \in W \\ \delta(q, X) = q'}} \big[\bigwedge_{\varphi \in X}\form_{\varphi, t} \land \bigwedge_{\varphi \in \mathsf{Max}(\Phi) \setminus X}\lnot \form_{\varphi, t} \big] &\text{if}\ x = (t, q), y = (t', q')\\
    \top_{T(t, \alpha)(t') > 0} &\text{if}\ x = (t, q), y = v_\mathsf{err}\\
    \top_{y = v_\mathsf{err}}&\text{otherwise}
    \end{cases}
\end{align*}
Now, the B-construction. We specify the transition relation of $A^\mathsf{NBA}_{\obs_i \lnot \Phi}$. The formula $\mathit{Tr}_{A^\mathsf{NBA}_{\obs_i \lnot \Phi}, (s_1, s'_1, q_1), (s_2, s'_2, q_2)}^{\alpha,t}$ expresses that $(s_1, s'_1, q_1)$ can transition to $(s_2, s'_2, q_2)$ under symbol $(\alpha, t)$.
\begin{align*}
    \mathit{Tr}_{A^\mathsf{NBA}_{\obs_i \lnot \Phi}, (s_1, s'_1, q_1), (s_2, s'_2, q_2)}^{\alpha,t} \defeq& \top_{t = s_2}\\
    & \land \top_{T(s_1, \alpha)(s_2) > 0}\\
    & \land \top_{\exists \alpha' \in \ACT^\AG .\; T(s'_1, \alpha')(s'_2) > 0\; \& \; \obs_i(\alpha, s_2) = \obs_i(\alpha', s'_2)}\\
    & \land \bigvee_{\substack{X \in W\\ D_{\lnot\Phi}(q_1, X, q_2)}} \big[\bigwedge_{\varphi \in X}\form_{\varphi, s'_1} \land \bigwedge_{\varphi \in \mathsf{Max}(\lnot\Phi) \setminus X}\lnot\form_{\varphi, s'_1}\big].
\end{align*}
Now, we encode the transition function of $A^\mathsf{DRA}_{\obs_i\lnot\Phi}$. Note that the notion of Safra tree can be defined completely outside of our real arithmetic formula. Similarly, we can perform most steps in the computation of the transition function outside our formula as well. Specifically, the only step which has to happen in the formula is step 3.

Let $\mathit{Saf}_{1\dots2,\alpha,t}^{\mathcal{T},\mathcal{T}'}$ be a formula that is defined as $\top$ if, on input symbol $(\alpha, t)$, from the Safra tree $\mathcal{T}$ we obtain the ordered tree $\mathcal{T}'$ (note that after these two steps the intermediary result need not satisfy the requirements of a Safra tree) by following steps 1 and 2 of the Safra tree computation we specified before. If not, the formula is defined as $\bot$. Similarly, let $\mathit{Saf}_{4\dots6,\alpha, t}^{\mathcal{T},\mathcal{T}'}$ be defined as $\top$ if under input symbol $(\alpha, t)$, from the ordered tree $\mathcal{T}$ we obtain the Safra tree $\mathcal{T}'$ by following steps 4 to 6 of the Safra tree computation we specified before. If not, the formula is defined as $\bot$.

Note that we need only consider ordered trees that are intermediary results in the Safra tree computation; these are ordered trees with at most as many nodes as the number of states of the original NBA, labelled with macrostates. Therefore the amount of such intermediary trees is again exponential wrt the number of states of the original NBA, and we can effectively enumerate them outside and inside of our formula. 

We will now define a formula $\mathit{Saf}_{3,\alpha,t}^{\mathcal{T},\mathcal{T}'}$, which states that on input symbol $(\alpha, t)$ from the ordered tree $\mathcal{T}$ we obtain the ordered tree $\mathcal{T}'$ by following step 3 of the Safra tree computation.
\begin{align*}
\mathit{Saf}_{3,\alpha,t}^{\mathcal{T},\mathcal{T}'} \defeq& \top_{\text{nodes of $\mathcal{T}$ and $\mathcal{T}'$ are the same}} \land \\
    & \bigwedge_{\substack{v \in \mathcal{T}'\\ (t_2, t'_2, q_2) \in S \times S \times Q_{\lnot\Phi}}} \big[\top_{(t_2, t'_2, q_2) \in K^{\mathcal{T}'}_{v}} \leftrightarrow \bigvee_{(t_1, t'_1, q_1) \in K^{\mathcal{T}}_v} \mathit{Tr}^{\alpha, t}_{A^\mathsf{NBA}_{\obs_i\lnot\Phi}, (t_1, t'_1, q_1), (t_2, t'_2, q_2)}\big].
\end{align*}

Using this, we can now finally specify the transition function of $A^\mathsf{DRA}_{\obs_i\lnot\Phi}$. The formula $\mathit{Tr}_{A^\mathsf{DRA}_{\obs_i\lnot\Phi}, \mathcal{T}, \mathcal{T}'}^{\alpha, t}$ expresses that $\delta^\mathsf{DRA}$ sends the Safra tree $\mathcal{T}$ to $\mathcal{T}$ on input symbol $(\alpha, t)$.
\begin{align*}
    \mathit{Tr}_{A^\mathsf{DRA}_{\obs_i\lnot\Phi}, \mathcal{T}, \mathcal{T}'}^{\alpha, t} \defeq \bigvee_{\text{intermediary ordered trees}\ \mathcal{T}_2, \mathcal{T}_3} \mathit{Saf}^{\mathcal{T}, \mathcal{T}_2}_{1\dots2, \alpha, t} \land \mathit{Saf}^{\mathcal{T}_2, \mathcal{T}_3}_{3, \alpha, t} \land \mathit{Saf}^{\mathcal{T}_3, \mathcal{T}'}_{4\dots6, \alpha, t}.
\end{align*}

This finishes the encodings of the transition functions in both the A- and B-constructions. We will now proceed to specify the computation of the probability in the product Markov chain, as we did in the procedure for probabilistic terms. We introduce variables $r_{s_1, q_1, q'_1}^{s_2,q_2, q'_2}$ for $s_1, s_2 \in S$, $q_1,q_2 \in Q_\Phi$, and $q'_1, q'_2 \in Q_{\obs_i\Phi}$. These will express the probability of transitioning from $(s_1, q_1, q'_1)$ to $(s_2, q_2, q'_2)$ in the product Markov chain. The formula $\mathit{Prod}_\beta$ again expresses this, with auxiliary binary variables $r^\mathit{Tr, A_\Phi}_{q_1, (\alpha, t), q_2}$ and $r^\mathit{Tr, A_{\obs_i\Phi}}_{q'_1, (\alpha, t), q'_2}$ which will numerically hold the truth value of the corresponding transition-formulas.

\begin{align*}
\mathit{Prod}_\beta  \defeq & \
\exists \big(r^\mathit{Tr, A_\Phi}_{q_1, (\alpha, t), q_2}\big)_{\substack{q_1, q_2 \in Q_\Phi \\ \alpha, t}} \exists \big(r^\mathit{Tr, A_{\obs_i\Phi}}_{q'_1, (\alpha, t), q'_2}\big)_{\substack{q'_1, q'_2 \in Q_{\obs_i\Phi} \\ \alpha, t}}.\\
&
\bigwedge_{\substack{q_1, q_2 \in Q_\Phi}} \bigwedge_{\alpha, t} [\mathit{Tr}^{\alpha, t}_{A_\Phi, q_1, q_2} \to  r^\mathit{Tr, A_\Phi}_{q_1, (\alpha, t), q_2} \approx 1] \land 
[\lnot \mathit{Tr}^{\alpha, t}_{A_\Phi, q_1, q_2} \to  r^\mathit{Tr, A_\Phi}_{q_1, (\alpha, t), q_2} \approx 0]\\
& \land \bigwedge_{\substack{q'_1, q'_2 \in Q_{\obs_i\Phi}}} \bigwedge_{\alpha, t} [\mathit{Tr}^{\alpha, t}_{A^\mathsf{DRA}_{\obs_i\lnot\Phi}, q'_1, q'_2} \to  r^\mathit{Tr, A_{\obs_i\Phi}}_{q'_1, (\alpha, t), q'_2} \approx 1] \\
&
\qquad \qquad \qquad \qquad \land 
[\lnot \mathit{Tr}^{\alpha, t}_{A^\mathsf{DRA}_{\obs_i\lnot\Phi}, q'_1, q'_2}\to  r^\mathit{Tr, A_{\obs_i\Phi}}_{q'_1, (\alpha, t), q'_2} \approx 0]\\
& \land \bigwedge_{\substack{s_1, s_2 \in S\\q_1, q_2 \in Q_\Phi\\q'_1, q'_2 \in Q_{\obs_i\Phi}}} r_{s_1, q_1, q'_1}^{s_2,q_2, q'_2} \approx \sum_{\alpha \in \ACT^\AG} r^\mathit{Tr, A_\Phi}_{q_1, (\alpha, s_2), q_2} \times r^\mathit{Tr, A_{\obs_i\Phi}}_{q'_1, (\alpha, s_2), q'_2} \times \\
& 
\qquad \qquad \qquad \qquad \qquad \qquad \qquad
T(s_1, \alpha)(s_2) \times \prod_{k \in \AG}r_{\beta(k), \obs_k(s_1), \alpha_k}.
\end{align*}

The reachability formula $\mathit{Reach}_v^w$ for states $v$ and $w$ of the product Markov chain is defined identically to how we did it in the construction for probabilistic terms, so we do not write it again. 

The formula $\mathit{Goal}_v$ needs to be changed to reflect that we now have a Streett acceptance condition, instead of a Rabin one.
\begin{align*}
\mathit{Goal}_v \defeq & [\bigwedge_{w \in S \times Q_\Phi \times Q_{\obs_i\Phi}} \mathit{Reach}_v^w \to \mathit{Reach}_w^v] \ \land \\
&
\quad \bigwedge_{(E,F) \in \mathsf{Acc}_\mathsf{prod}}[\bigvee_{w \in S \times F} \mathit{Reach}^{w}_{v}
\lor \bigwedge_{w \in S \times E} \lnot\mathit{Reach}_v^w]
\end{align*}

The formula $\mathit{Sol}$ is constructed identically to how we did it for probabilistic terms. So we can now conclude the model checking procedure by specifying $\mathit{Eqn}_{\obs_i\Phi, s}$:
$$
\mathit{Eqn}_{\obs_i\Phi, s} \defeq \exists (r_v^w, r^\mathsf{sol}_{v})_{v,w \in S \times Q_\Phi \times Q_{\obs_i\Phi}} [\mathit{\mathit{Prod}_\beta} \land \mathit{Sol} \land r_{\obs_i\Phi} \approx r^\mathsf{sol}_{s, q^\ast_\Phi, q^\ast_{\obs_i\Phi}}],
$$
where $q^\ast_\Phi$ and $q^\ast_{\obs_i\Phi}$ are the initial states of, respectively, $A_\Phi$ and $A^\mathsf{DSA}_{\obs_i\Phi}$.

\section{Complexity}
Sentences of real arithmetic can be verified in exponential space w.r.t. the number of quantifiers in the sentence, and logarithmic\footnote{Technically, it requires powers of logarithms, i.e. $(\log n)^{O(1)}$, but this matters not for our analysis.} space w.r.t. the size of the quantifier-free part of the sentence - see Theorem 14.14 and Remark 13.10 of ``Algorithms in Real Algebraic Geometry'' (2006) by Basu, Pollack and Roy.

It should be clear from the construction of our sentence that the size of the construction $\mathit{Eqn}_{\DDD_{\beta, i}(\Phi), s}$ dominates the other parts of the construction. We claim that both the A- and B-constructions produce DSAs of size double exponential w.r.t. $\Phi$.

For the A-construction, this occurs since the DSA constructed from an LTL formula will generally be of size double exponential in the LTL formula's size \begin{nm}TODO: Add reference that says that LTL to DSA gives a double exponential blowup.\end{nm}.
For the B-construction, we have that the NBA we construct is of size exponential in $\Phi$. \begin{nm}TODO: Add reference that says that LTL to NBA gives a single exponential blowup.\end{nm} Applying Safra's construction, we then incur another exponential blowup, landing us with a DRA (and therefore also DSA) double exponential in size w.r.t. $\Phi$. The final product DSA and Markov chain are thus also of size double exponential wrt $\Phi$.

It is easily verified that no part of the real arithmetic formulas requires another exponential blowup on top of this number, so we land with a real arithmetic sentence that is double exponential in size wrt $\Phi$.
The number of quantifiers in the sentence can be verified to also be double exponential w.r.t. the size of $\Phi$, as there are e.g. quantifiers for each Safra tree.

The dependence of these factors on the system $\MMM$ is different. The sentence is of size exponential wrt the system, as it has e.g. conjuncts for each $\alpha \in \ACT^\AG$, as well as disjuncts for each (intermediary) Safra tree, the number of which is exponential wrt the system (but double exponential wrt the formula). The number of quantifiers is similarly exponential wrt the system, as there is at least one quantifier for each Safra tree.

Summarising, we have that the size of the real arithmetic sentence is double exponential wrt the formula, and single exponential wrt the system. The number of quantifiers in it is also double exponential wrt the formula, and single exponential wrt the system. By the complexity result for checking real arithmetic we mentioned in the beginning, we can thus conclude that our overall model checking procedure runs in space triple exponential wrt the formula, and double exponential wrt the system. Thus we incur a single exponential blowup in the complexity of the procedure wrt the system size, when comparing to the procedure for PSL.

\end{document}